\documentclass{article} 
\usepackage{iclr2027_conference,times}

\usepackage{amsmath,amsfonts,bm}

\def\eqref#1{equation~\ref{#1}}

\def\1{\bm{1}}

\DeclareMathAlphabet{\mathsfit}{\encodingdefault}{\sfdefault}{m}{sl}
\SetMathAlphabet{\mathsfit}{bold}{\encodingdefault}{\sfdefault}{bx}{n}

\usepackage{url}

\usepackage{graphicx}
\usepackage{xspace}
\usepackage[utf8]{inputenc} 
\usepackage[T1]{fontenc}    
\usepackage[dvipsnames,table]{xcolor}
\usepackage[
  colorlinks=true,
  linkcolor=NavyBlue,
  citecolor=TealBlue,
  urlcolor=RoyalBlue
]{hyperref}
\usepackage{bbm}
\usepackage{inconsolata}    
\usepackage{amsmath}      
\usepackage{amssymb}      
\usepackage{amsthm}       

\usepackage{enumitem}     
\usepackage{booktabs}     
\usepackage{placeins}    
\usepackage{needspace}   
\usepackage{algorithm}    
\usepackage{natbib}       
\usepackage{threeparttable} 
\usepackage{tabularx}
\usepackage{makecell}
\usepackage{subcaption}
\usepackage{titletoc}
\usepackage[breakable]{tcolorbox}
\usepackage{changepage}
\usepackage{fvextra}
\DeclareUnicodeCharacter{2192}{\ensuremath{\rightarrow}}
\DeclareUnicodeCharacter{2212}{\ensuremath{-}}
\tcbuselibrary{skins}

\definecolor{PromptBackground}{RGB}{248,249,251}
\definecolor{PromptFrame}{RGB}{200,208,220}
\definecolor{PromptSystem}{RGB}{37,82,140}
\definecolor{PromptUser}{RGB}{34,120,78}
\definecolor{PromptVariable}{RGB}{190,60,30}

\newcommand{\PromptTextSetup}{%
  \setlength{\parindent}{0pt}%
  \setlength{\parskip}{0.3em}%
  \setlist[itemize]{leftmargin=1.4em,itemsep=0.2em,topsep=0.2em,
    parsep=0pt,partopsep=0pt}%
  \setlist[enumerate,1]{leftmargin=1.9em,labelwidth=1.2em,labelsep=0.4em,
    align=left,itemsep=0.2em,topsep=0.2em,parsep=0pt,partopsep=0pt}%
}

\tcbset{
  paper prompt/.style={
    enhanced,breakable,
    width=\linewidth,
    colback=PromptBackground,colframe=PromptFrame,
    boxrule=0.6pt,arc=4pt,
    boxsep=0pt,left=10pt,right=10pt,top=12pt,bottom=8pt,
    fontupper=\footnotesize\sffamily\raggedright,
    fonttitle=\footnotesize\sffamily\bfseries,
    coltitle=white,
    titlerule=0pt,toptitle=3pt,bottomtitle=3pt,
    lefttitle=10pt,righttitle=10pt,
    attach boxed title to top left={xshift=10pt,yshift=-7pt},
    boxed title style={
      boxrule=0pt,arc=2pt,boxsep=0pt,
      left=6pt,right=6pt,top=3pt,bottom=3pt
    },
    before skip=9pt,after skip=4pt,
    pad at break*=6pt,
    lines before break=4,
    before upper=\PromptTextSetup
  },
  system prompt/.style={
    paper prompt,colbacktitle=PromptSystem,
    before upper={\PromptTextSetup\colorlet{PromptAccent}{PromptSystem}},
    title={System Prompt: #1},
    title after break={System Prompt: #1 (continued)}
  },
  user prompt/.style={
    paper prompt,colbacktitle=PromptUser,
    before upper={\PromptTextSetup\colorlet{PromptAccent}{PromptUser}},
    title={User Prompt: #1},
    title after break={User Prompt: #1 (continued)}
  },
  prompt excerpt/.style={
    enhanced,colback=white,colframe=white,boxrule=0pt,arc=0pt,
    borderline west={1.5pt}{0pt}{PromptUser!45!PromptFrame},
    boxsep=0pt,left=7pt,right=5pt,top=5pt,bottom=5pt,
    before skip=4pt,after skip=4pt,
    fontupper=\footnotesize\sffamily\raggedright
  }
}

\newcommand{\PromptHeading}[1]{%
  \par\addvspace{0.6\baselineskip}%
  {\bfseries\color{PromptAccent}#1\par}%
  \nobreak\smallskip\nobreak
}
\newcommand{\PromptSubheading}[1]{%
  \par\addvspace{0.4\baselineskip}%
  {\bfseries #1\par}%
  \nobreak\smallskip\nobreak
}
\newcommand{\PromptPlaceholder}[1]{%
  \textcolor{PromptVariable}{\texttt{\{#1\}}}%
}

\DefineVerbatimEnvironment{PromptCode}{Verbatim}{
  fontsize=\scriptsize,
  samepage=true,
  breaklines=true,breakanywhere=false,
  breaksymbolleft={},breaksymbolright={},
  frame=leftline,rulecolor=\color{PromptFrame},
  framerule=1pt,framesep=6pt
}

\usepackage[noend]{algorithmic}
\definecolor{phasecolor}{HTML}{0ABAB5}   
\newcommand{\phase}[1]{\vspace{2pt}{\color{phasecolor}$\blacktriangleright$\ \textbf{#1}}}
\newcommand{\sidenote}[1]{\hfill{\color{black!55}\footnotesize\itshape #1}}
\algsetup{indent=1.2em}

\theoremstyle{plain}

\newtheorem{proposition}{Proposition}

\theoremstyle{definition}

\theoremstyle{remark}

\iclrfinalcopy

\title{AlphaBet: Reinforcement Learning for \\
LLM Event Forecasting Agents}

\newcommand{\projectdataurl}{https://huggingface.co/collections/liqinye/alphabet}
\newcommand{\projectcodeurl}{https://github.com/AlphaBet-Agent/AlphaBet}
\newcommand{\projectwebsiteurl}{https://alphabet-agent.github.io/}

\definecolor{ProjectLink}{HTML}{234F9F}
\newcommand{\projectresource}[3]{%
    \begingroup
    \edef\resourceurl{#1}%
    \def\resourcelabel{%
        \mbox{%
            \makebox[1.45em][c]{%
                \raisebox{-0.2\height}{%
                    #2%
                }%
            }%
            \hspace{0.35em}\textcolor{ProjectLink}{\texttt{#3}}%
        }%
    }%
    \ifx\resourceurl\empty
        \resourcelabel
    \else
        \href{\resourceurl}{\resourcelabel}%
    \fi
    \endgroup
}
\newcommand{\projectresources}{%
    \raisebox{-5pt}[\height][\depth]{%
    \normalfont\sffamily\fontsize{9}{11}\selectfont
    \projectresource{\projectdataurl}{\includegraphics[height=1.25em]{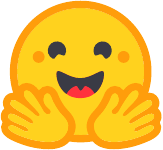}}{Data \& Models}%
    \hspace{1.5em}%
    \projectresource{\projectcodeurl}{\includegraphics[height=1.25em]{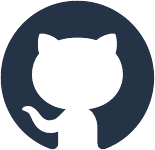}}{Code}%
    \hspace{1.5em}%
    \projectresource{\projectwebsiteurl}{\includegraphics[height=1.25em]{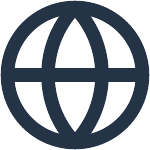}}{Website}}%
}

\author{%
 Liqin Ye\textsuperscript{1\textdagger}, 
 Rongzhi Zhang\textsuperscript{2}, 
 Xin Liu\textsuperscript{2},
 Zheng Li\textsuperscript{2},
 \textbf{Yuan He}\textsuperscript{2},\\
 \textbf{Qingyu Yin}\textsuperscript{2},
 \textbf{Priyanka Nigam}\textsuperscript{2},
 \textbf{Chao Zhang}\textsuperscript{1}\\[3pt]
 \textsuperscript{1}Georgia Institute of Technology,
 \textsuperscript{2}Amazon\\[7pt]
  \makebox[\dimexpr\textwidth-2\tabcolsep\relax][l]{\projectresources}
}

\definecolor{ours}{RGB}{226,238,252}
\newcommand{\ours}{\textsc{AlphaBet}\xspace}
\makeatletter
\newcommand{\metricci}[2]{%
  \begin{tabular}[t]{@{}c@{}}
    #1\\[-4pt]
    \noalign{\global\let\CT@row@color\relax}%
    {\scriptsize[#2]}
  \end{tabular}%
}
\newcommand{\modelcell}[1]{%
  \raisebox{-0.7ex}{#1}%
}
\providecolor{ours}{RGB}{226,238,252}
\providecolor{FDLightBand}{HTML}{EFF3F6}

\providecommand{\metricci}[2]{%
  \shortstack{#1\\[-1pt]{\tiny\color{gray}[#2]}}%
}
\newcommand{\metricsolo}[1]{%
  \begin{tabular}[t]{@{}c@{}}
    #1\\[-5pt]
    \noalign{\global\let\CT@row@color\relax}%
    {\scriptsize\phantom{[0.000, 0.000]}}
  \end{tabular}%
}
\makeatother
\newcommand{\sig}[2]{%
  #1\rlap{$^{#2}$}%
}
\makeatletter
\newcommand{\blfootnote}[1]{%
  \begingroup
  \renewcommand{\thefootnote}{}%
  \renewcommand{\@makefntext}[1]{\noindent##1}%
  \begin{NoHyper}\footnotetext{#1}\end{NoHyper}%
  \endgroup
}
\makeatother

\begin{document}

\maketitle
\lhead{Preprint}
\blfootnote{Corresponding author(s): \texttt{\{liqiny, chaozhang\}@gatech.edu}}
\blfootnote{\textsuperscript{\textdagger}Work done during an internship at Amazon.}
\begin{abstract}
Forecasting an unresolved event requires finding relevant evidence and deciding how strongly it supports each outcome. Yet outcome-based training of Large Language Model (LLM) forecasters often begins with preassembled evidence at a single prediction date, leaving evidence acquisition outside the learning loop. We introduce \ours, an agentic reinforcement learning (RL) framework that trains a single LLM policy to search, reason, and report probabilities across multiple pre-resolution dates. Each dated forecast is scored by a proper scoring rule against the eventual outcome, allowing one resolved event to supervise decisions under different information states. We decompose forecasting regret into an evidence-acquisition gap and a probability-reporting gap, and show that Group Relative Policy Optimization (GRPO)'s standard-deviation normalization can distort the reporting target under multi-class proper scores. On GLM-4.5-Air, \ours reduces the Brier score from 0.731 to 0.626 on Forecast-Dojo and from 0.612 to 0.528 on FutureX. Controlled ablations attribute these gains to multi-date training and show that they come from improvements in both information acquisition and probability reporting. Our empirical results validate that GRPO's std normalization degrades proper-scoring-rule performance while leaving top-outcome accuracy largely unchanged.
\end{abstract}

\begin{figure}[!ht]
    \centering
    \includegraphics[width=0.49\linewidth]{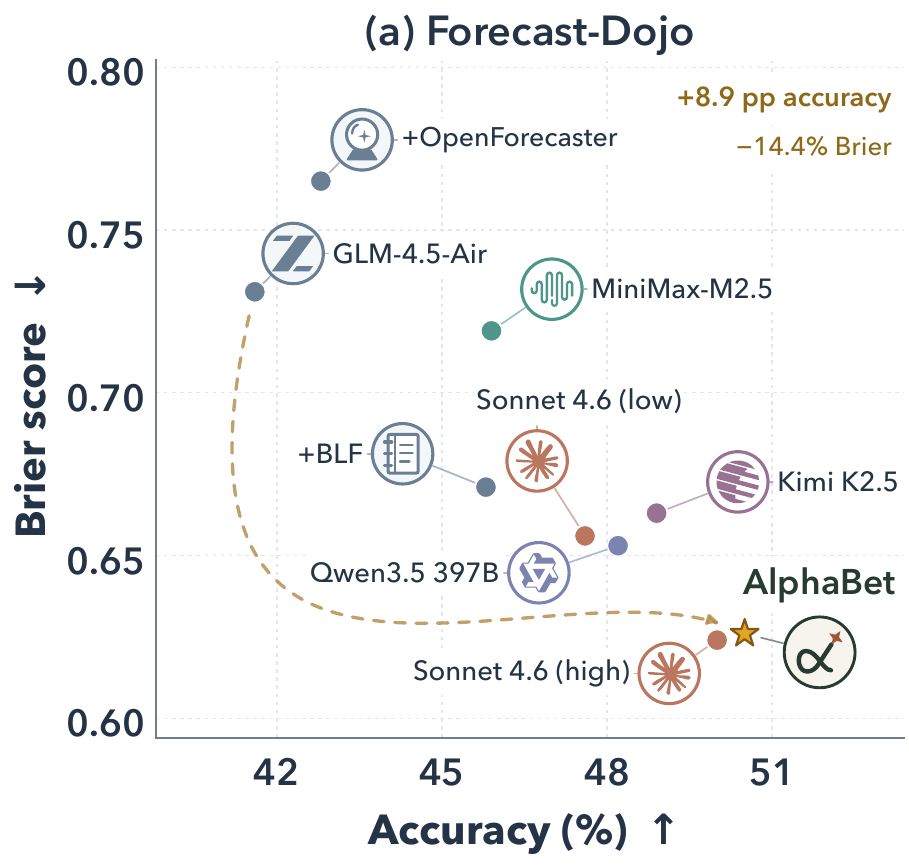}%
    \hfill
    \includegraphics[width=0.49\linewidth]{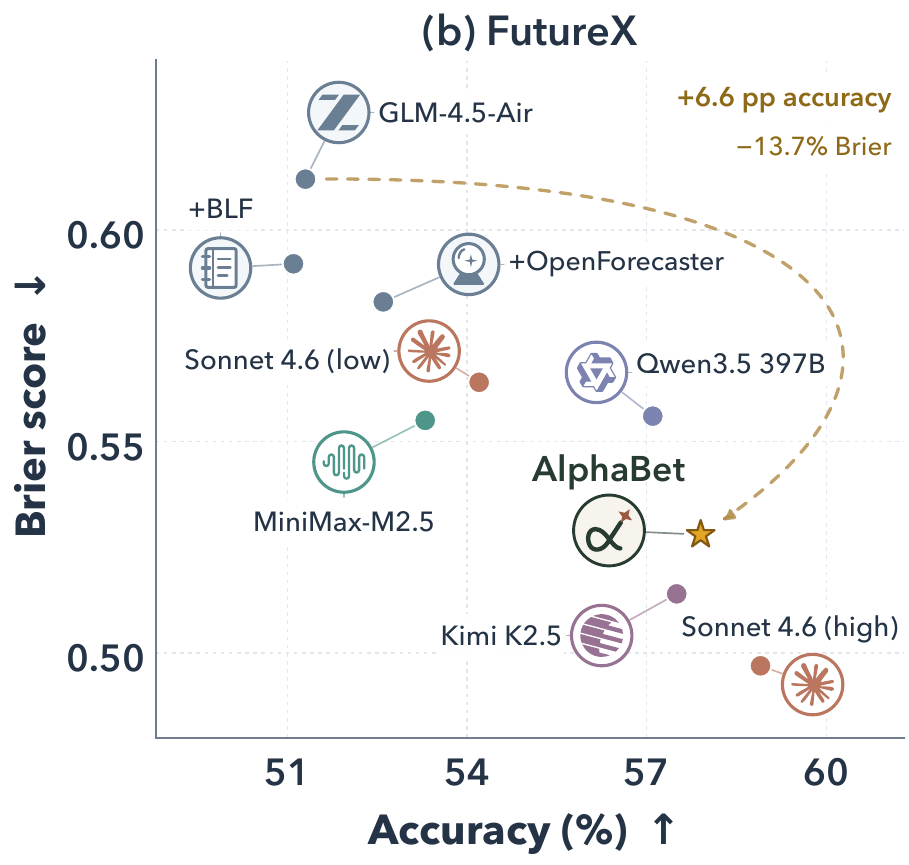}
    \vspace{-0.7em}
    \caption{\textbf{\ours improves forecasting on Forecast-Dojo and FutureX (March 2026 to May 2026).} Dashed arrows connect GLM-4.5-Air to \ours{}$_\text{{\scriptsize\color{black}memory-free}}$. The lower right corresponds to higher accuracy and lower Brier score. Accuracy gains are in percentage points, and Brier reductions are relative to the base model. Table~\ref{tab:main} reports full results and confidence intervals.}
    \label{fig:teaser_results}
\end{figure}

\begingroup
\setlength{\parskip}{4.5pt}
\section{Introduction}
\label{sec:intro}

Many real-world event forecasting questions, such as policy decisions or economic indicators, remain unresolved for weeks or months, while the relevant evidence continues to accumulate. Human forecasters have long practiced \emph{longitudinal forecasting}, revisiting the same unresolved event and updating their beliefs as information evolves~\citep{mellers2014psychological,atanasov2020small}. LLM agents can retrieve and reason over evidence and produce probabilistic forecasts~\citep{futurex,forecastbench}. Recent LLM forecasting benchmarks have begun to assess this temporal dimension more directly, evaluating their forecasts at different points before resolution~\citep{futuresim,yuan2025language,forecastdojo,livemacro}.

Yet existing outcome-based forecasting methods still train only part of
this process. They commonly optimize probability reporting from
preassembled evidence at a single prediction time, leaving evidence
acquisition outside the trained policy and providing no learning signal
for how forecasts should evolve as information changes
~\citep{approachHuman,outcomeRL,openForecaster,scott2026forecasting}.
Forecast-Dojo~\citep{forecastdojo} provides a replayable, temporally indexed forecasting research environment that makes this broader learning problem accessible. This motivates our central question: \emph{how can outcome feedback train evidence-seeking and probability-reporting behavior across changing information states?}

We introduce \ours (Figure~\ref{fig:overview}), an outcome-based RL framework for agentic event forecasting. Using Forecast-Dojo's temporally indexed environment, a single tool-using policy revisits each training event at multiple pre-resolution dates, actively searches and reads the evidence
available at each date, and reports a distribution over outcomes.
Each dated forecast receives proper-score feedback against the eventual
outcome, so one resolved event supervises decisions under multiple
information states. Cross-date memory is optional: the memory-free
variant starts each dated forecast afresh, while the memory-on variant
carries a belief notebook across dates. We show that the resulting forecasting regret decomposes into an information-acquisition gap and a probability-reporting gap (Proposition~\ref{prop:longitudinal_properness}). For policy optimization, we use GRPO~\citep{shao2024deepseekmath} without within-group standard-deviation normalization~\citep{outcomeRL,openForecaster,liu2025understandingr1zeroliketrainingcritical}. We provide a proper-scoring interpretation of this choice.
Extending the binary analysis of
\citet{bereket2025uncalibratedreasoninggrpoinduces}, we show that, at a
fixed acquired context, mean-centering preserves the expected reporting
gradient induced by a strictly proper score, whereas standardization can
reweight outcome classes and change the elicited probability target.

Empirically, \ours improves forecasting performance on both Forecast-Dojo and FutureX. Our strongest variant improves GLM-4.5-Air from 41.6\% to 50.5\% accuracy and from 0.731 to 0.626 Brier on Forecast-Dojo evaluation set, and from 51.3\% to 57.9\% accuracy and from 0.612 to 0.528 Brier on FutureX, outperforming forecasting harnesses and prior RL-based training methods while remaining competitive with substantially larger off-the-shelf models. Multi-date training consistently improves over single-date training. Controlled evidence--reader swaps show improvements in both evidence acquisition and probability reporting. Finally, GRPO ablations support our proper-scoring analysis: within-group standardization degrades probabilistic forecast quality, while removing it preserves the intended training signal.

Our contributions are:
\begin{itemize}[leftmargin=1.2em, itemsep=0.2em, topsep=0.1em, parsep=0pt, partopsep=0pt]
    \item We introduce \ours, an outcome-based agentic RL framework that trains a single policy to actively acquire evidence and report probabilistic forecasts across multiple pre-resolution dates.
    
    \item We characterize outcome-based forecasting across evolving information states: a single realized outcome can properly supervise multiple dated forecasts, forecasting regret decomposes into information-acquisition and probability-reporting gaps, and GRPO normalization can alter the intended multi-class probability target.
    
    \item We establish empirically that multi-date training improves forecasting across two benchmarks, with gains arising from both better evidence acquisition and better probability reporting. GRPO ablations further validate our analysis of normalization and probabilistic forecast quality.
\end{itemize}

\par
\endgroup

\begingroup
\setlength{\parskip}{3pt}
\vspace{-0.5em}
\section{Preliminaries}
\label{sec:preliminaries}

\vspace{-0.4em}
\subsection{Longitudinal Event Forecasting}
\label{sec:longitudinal_forecasting}

\paragraph{Forecasting setup.}
\leavevmode\clubpenalty=0\relax%
Each question $Q$ describes an event in natural language and specifies a finite set of mutually exclusive and exhaustive outcomes $\mathcal{Y}=\{1,\ldots,K\}$. Its resolution criterion specifies how the realized outcome $Y\in\mathcal{Y}$ is determined. During the question's active period
$[\tau_{\mathrm{start}},\tau_{\mathrm{close}}]$, the forecaster revisits $Q$
before resolution at scheduled dates
$\tau_{\mathrm{start}}\leq\tau_1<\cdots<\tau_T\leq\tau_{\mathrm{close}}$.
We index these visits by $t=1,\ldots,T$ and call each a
\emph{forecast step}: step $t$ occurs on date $\tau_t$ and produces
a probability report $p_{Q,t}\in\Delta^{K-1}$ over $\mathcal{Y}$.
The ordered sequence of steps forms a \emph{forecast episode}.

\par
\endgroup

\paragraph{Evolving information corpus.}
The question remains fixed across these steps, but the evidence available to the forecaster evolves. At date $\tau_t$, it can consult $\mathcal{I}_{\leq\tau_t}$, the corpus of information available no later than that date. Sources released later are excluded~\citep{zou2022forecasting,approachHuman}. As new information arrives, the admissible corpus expands:
$\mathcal{I}_{\leq\tau_t}\subseteq\mathcal{I}_{\leq\tau_{t+1}}$. The forecaster typically observes only part of this corpus. In our setting, an LLM agent uses tools to retrieve and analyze evidence, interleaving these calls with reasoning. We write $H_t$ for the Search and Read results from step $t$ only, excluding the agent's reasoning, code, and notes.

\paragraph{Forecast targets.}
One question's forecasts share the same eventual outcome $Y$, but the relevant information changes with the forecast date. We distinguish the posterior supported by the full dated corpus from that supported by the current evidence and incoming notebook $M_{t-1}$, which is empty when memory is disabled:

\noindent\begin{minipage}[b]{0.48\linewidth}
\vspace{-0.5em}
\begin{equation}
    q_{Q,t}(k)
    \triangleq
    \Pr\!\left(Y=k\mid Q,\tau_t,\mathcal{I}_{\leq\tau_t}\right)
    \label{eq:dated_corpus_target}
\end{equation}
\end{minipage}\hfill
\begin{minipage}[b]{0.48\linewidth}
\vspace{-0.5em}
\begin{equation}
    q^H_{Q,t}(k)
    \triangleq
    \Pr\!\left(Y=k\mid Q,\tau_t,H_t,M_{t-1}\right)
    \label{eq:context_target_main}
\end{equation}
\end{minipage}

$q_{Q,t}$ is the \emph{dated-corpus target}, the Bayes-optimal forecast given all admissible information available by $\tau_t$. $q^H_{Q,t}$ is the \emph{acquired-context target}, the Bayes-optimal forecast given $H_t$ and $M_{t-1}$. The agent reports $p_{Q,t}$ based on this acquired context. These three quantities separate what is inferable from the dated corpus, what is inferable from the agent's current evidence and retained memory, and what the agent ultimately reports.

\subsection{Proper Scoring Rules}
\label{sec:proper_scoring}

Evaluating a probabilistic forecast requires more than checking its most
likely outcome: forecasts that choose the same outcome can assign very
different probabilities. Proper scoring rules score these probabilities
directly and make truthful reporting optimal in expectation
~\citep{gneiting2007strictly}, providing both an evaluation criterion and
a training signal from the reported distribution and eventual outcome
without observing the latent target.
Formally, $S$ is \emph{strictly proper} if, for every outcome distribution $q$
and report $p$,
\vspace{-0.2em}
\begin{equation}
    \mathbb{E}_{Y\sim q}[S(q,Y)]
    \geq
    \mathbb{E}_{Y\sim q}[S(p,Y)],
    \qquad
    \text{with equality if and only if }p=q.
    \label{eq:strict_properness}
    \vspace{-0.2em}
\end{equation}
Conditioned on the acquired context, strict properness therefore makes
$q^H_{Q,t}$ the unique optimal report in expectation. Relative to the
dated-corpus target $q_{Q,t}$, this distinguishes an
\emph{information gap} between $q_{Q,t}$ and $q^H_{Q,t}$ from a
\emph{reporting gap} between $q^H_{Q,t}$ and $p_{Q,t}$.

\section{Method}
\label{sec:method}

\begin{figure}[!t]
      \centering
      \includegraphics[width=\linewidth]{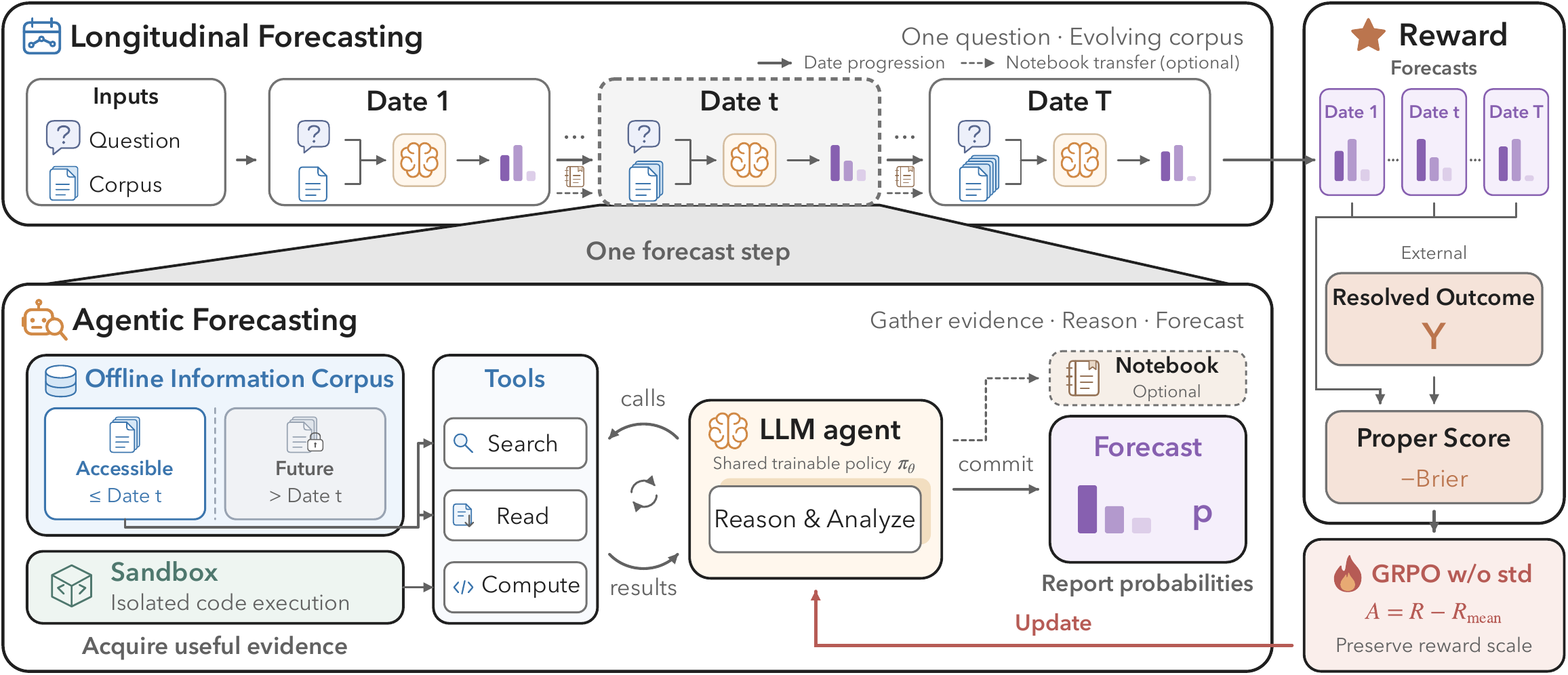}
      \vspace{-1.5em}
      \caption{
          \textbf{Overview of \ours.} One shared policy searches, reads, and reasons over the corpus available at each forecast date before reporting probabilities. Each report is scored against the eventual outcome, which is hidden during rollout. Training uses per-date feedback with mean-centered GRPO; a notebook optionally carries context between dates.
      }
      \vspace{-1.2em}
       \label{fig:overview}
\end{figure}
\subsection{RL Environment for Longitudinal Forecasting}
\label{sec:environment}

We train a single policy to acquire evidence and report probabilities across
multiple pre-resolution dates, using the eventual outcome as feedback at
each date. Following Forecast-Dojo~\citep{forecastdojo}, the agent gathers and reasons over evidence at each date before reporting a forecast, optionally carrying a belief notebook forward.
\vspace{-1em}
\paragraph{Agent loop within a forecast step.}
Forecast step $t$ begins with the question $Q$, the forecast date $\tau_t$,
and the belief notebook $M_{t-1}$ from the preceding step when external memory is enabled. The agent then interleaves reasoning with three research capabilities for
gathering and analyzing evidence:
\begin{itemize}[leftmargin=1.5em, itemsep=0.1em, topsep=-0.3em]
    \item \textsc{Search} returns the top-$k$ relevant article snapshots for an agent-generated query from the date-restricted corpus $\mathcal{I}_{\leq\tau_t}$ via \texttt{search(query,top-k)}.
    \item \textsc{Read} retrieves the full text of a previously found article, via \texttt{scrape(article\_id)}.
    \item \textsc{Compute} executes model-generated Python code for numerical analysis, aggregation, base-rate estimation, or simulation, via \texttt{python(code)}.
\end{itemize}
Each tool call result extends the agent's context and can prompt further reasoning or another call. Search and Read results from this step form $H_t$, which the agent analyzes alongside any incoming notebook. The agent completes the step by committing $p_t\in\Delta^{K-1}$. When external memory is enabled, the agent additionally produces an updated belief notebook $M_t$, which summarizes its current belief, unresolved questions, and accumulated evidence for use at the next forecast step.

\paragraph{Progression between forecast steps.}
After completing the forecast at $\tau_t$, the environment advances to the next scheduled date $\tau_{t+1}$, expanding the admissible information from $\mathcal{I}_{\leq\tau_t}$ to
$\mathcal{I}_{\leq\tau_{t+1}}$. When external memory is disabled, each forecast step starts without agent-generated
belief notebook from earlier dates. We refer to this as \emph{memory-free forecasting}. When external memory is enabled, the belief notebook $M_t$ is carried into the next forecast step, yielding \emph{memory-on forecasting}. 
We repeat this process through the final scheduled date $\tau_T$ to complete the episode.

\subsection{Training Rewards and Objective}
\label{sec:proper_objective}

\paragraph{Training reward.}
Our reward for forecast step $t$ is defined as a sum of a forecasting score
and a format reward:
\vspace{-0.15em}
\begin{equation}
    R_t
    =
    S(p_t,Y)
    +
    R_t^{\mathrm{format}},
    \qquad
    R_t^{\mathrm{format}}
    =
    \lambda_{\mathrm{ans}}\bigl(f_t^{\mathrm{ans}}-1\bigr)
    +
    \lambda_{\mathrm{tool}}\bigl(f_t^{\mathrm{tool}}-1\bigr).
    \label{eq:training_reward}
\end{equation}
Our main experiments use the negative Brier score
(Section~\ref{sec:proper_scoring}), while the GRPO normalization ablation
also considers the logarithmic score. Exact properness holds for the
ideal log score under finite expectation, but not for its probability-floored
implementation on the full simplex (Appendix~\ref{app:properness_scope}). The format reward enforces a valid interface: $f_t^{\mathrm{ans}}\in[0,1]$ checks the answer format and, when memory is enabled, the notebook structure and its agreement with the forecast, while $f_t^{\mathrm{tool}}\in[0,1]$ checks whether tool calls use names recognized by the environment. We provide full definitions, coefficients, and score implementation in Appendix~\ref{app:reward_implementation}.

\paragraph{Training objective.}
Let $\mathcal D$ denote the training distribution over resolved forecasting
questions. We maximize expected cumulative reward across forecast dates,
with expectation over policy rollouts:
\begin{equation}
    \underset{\theta}{\max}\;\underbrace{\mathbb{E}_{Q\sim\mathcal D}\!\left[\sum_{t=1}^{T_Q}R_{Q,t}\right]}_{J_{\mathrm{train}}(\theta)}
    =
    \underset{\theta}{\max}\left\{
    \underbrace{\mathbb{E}_{Q\sim\mathcal D}\!\left[\sum_{t=1}^{T_Q}S(p_{Q,t},Y_Q)\right]}_{J_{\mathrm{forecast}}(\theta)}
    +
    \underbrace{\mathbb{E}_{Q\sim\mathcal D}\!\left[\sum_{t=1}^{T_Q}R_{Q,t}^{\mathrm{format}}\right]}_{J_{\mathrm{format}}(\theta)}
    \right\}.
    \label{eq:longitudinal_objective}
\end{equation}
The format term satisfies $J_{\mathrm{format}}(\theta)\leq 0$ and vanishes for valid outputs. We assume there exists a policy that maximizes $J_{\mathrm{forecast}}$ while satisfying all format checks. Then, every maximizer of $J_{\mathrm{train}}$ is also forecasting-optimal, $\max_\theta J_{\mathrm{train}}(\theta)=\max_\theta J_{\mathrm{forecast}}(\theta)$. Appendix~\ref{app:format_reward} gives the proof.

Improvements in $J_{\mathrm{forecast}}$ can arise through two complementary
capabilities: acquiring relevant information and translating that information
into accurate probabilistic forecasts. Proposition~\ref{prop:longitudinal_properness}
shows that our objective captures both.

\begin{proposition}[\textbf{Properness of the agentic forecasting objective}]
\label{prop:longitudinal_properness}
Assume there is no outcome leakage beyond the admissible dated information, as formalized in Appendix~\ref{app:properness}.
Let $S$ be a strictly proper score with finite expected score. For the targets $q_{Q,t}$ and $q^H_{Q,t}$ defined in Section~\ref{sec:longitudinal_forecasting}, define
\vspace{-0.2em}
\begin{equation}
    D_S(q,p)\triangleq\mathbb{E}_{Y\sim q}[S(q,Y)-S(p,Y)],\qquad
    J^\star\triangleq\mathbb{E}_{Q\sim\mathcal D}\!\left[\sum_{t=1}^{T_Q}S(q_{Q,t},Y_Q)\right].
    \label{eq:forecast_regret_and_optimum}
\vspace{-0.2em}
\end{equation}
Here $D_S(q,p)$ is the expected proper-score discrepancy between $q$ and $p$.
$J^\star$ is the oracle forecasting value obtained by reporting the dated-corpus target $q_{Q,t}$ at every step.
Then
\begin{equation}
    \begin{aligned}
        \underset{\theta}{\arg\max}\;J_{\mathrm{forecast}}(\theta)
        &=\underset{\theta}{\arg\min}\;\bigl[J^\star-J_{\mathrm{forecast}}(\theta)\bigr] \\[-0.7em]
        &=\underset{\theta}{\arg\min}\;\mathbb{E}_{Q\sim\mathcal D}\!\left[\sum_{t=1}^{T_Q}\left(\underbrace{D_S(q_{Q,t},q^H_{Q,t})}_{\text{information gap}}+\underbrace{D_S(q^H_{Q,t},p_{Q,t})}_{\text{reporting gap}}\right)\right].
    \end{aligned}
    \label{eq:forecasting_gap_decomposition}
\end{equation}
A policy has zero total gap if and only if $p_{Q,t}=q^H_{Q,t}=q_{Q,t}$ at every forecast step, almost surely.
\end{proposition}

Under the assumption above, the auxiliary format reward does not alter the
forecasting optimum. We therefore analyze $J_{\mathrm{forecast}}$ directly. Proposition~\ref{prop:longitudinal_properness} shows that maximizing \(J_{\mathrm{forecast}}\) is exactly equivalent to minimizing the expected sum of the information and reporting gaps. This decomposition identifies two distinct sources of suboptimality, but does not imply that joint optimization reduces each separately: gains may come from better information acquisition, better reporting, or both. For a fixed acquired context, strict properness favors \(p_{Q,t}=q^H_{Q,t}\), while improved evidence acquisition can reduce the information gap. We provide the full proof of Proposition~\ref{prop:longitudinal_properness} in Appendix~\ref{app:properness}.

\subsection{Policy Optimization: GRPO for Calibrated Forecasting}
\label{sec:calibration_grpo}
We use GRPO for policy optimization but omit its within-group
standard-deviation normalization. This no-std variant has been adopted in
recent forecasting work~\citep{outcomeRL,openForecaster} and also appears
in Dr.~GRPO~\citep{liu2025understandingr1zeroliketrainingcritical} in a
different optimization setting. \citet{bereket2025uncalibratedreasoninggrpoinduces}
show that standardization can induce overconfidence for binary stochastic
outcomes. We extend this analysis to multi-class proper scoring rules:
at a fixed acquired context, mean-centering preserves the expected
score-only reporting gradient up to a positive factor, whereas
standardization can reweight outcome classes and shift the elicited
probability target.

\paragraph{Vanilla GRPO.}
GRPO compares multiple rollouts in a group and constructs relative advantages
without a learned value function. For rewards $\{R^{(i)}\}_{i=1}^{G}$, the
standard normalized advantage is
\vspace{-0.3em}
\begin{equation}
    A_{\mathrm{std}}^{(i)}
    =
    \frac{R^{(i)}-\bar R}{\sigma_R+\epsilon},
    \qquad
    \bar R=\frac{1}{G}\sum_{j=1}^{G}R^{(j)},
    \label{eq:standard_grpo}
    \vspace{-0.3em}
\end{equation}
where $\sigma_R$ is the within-group standard deviation and $\epsilon>0$
ensures numerical stability.

\paragraph{Removing reward standardization.}
Vanilla GRPO both mean-centers and standardizes group rewards.
Mean-centering provides a relative baseline, whereas standardization
additionally removes reward scale. In probabilistic forecasting, this
can distort the proper-score target because the normalization scale
depends on both the current policy and the realized outcome, rescaling
different outcome classes differently. We therefore retain mean-centering
but remove std normalization:
\begin{equation}
    A_{Q,t}^{(i)}
    =
    R_{Q,t}^{(i)}-\bar R_{Q,t},
    \qquad
    \bar R_{Q,t}=\frac{1}{G}\sum_{j=1}^{G}R_{Q,t}^{(j)}.
    \label{eq:no_std_advantage}
\end{equation}
The following proposition makes the distinction precise.

\begin{proposition}[\textbf{Proper-score targets under GRPO normalization}]
\label{prop:grpo_std}
Fix a question $Q$, forecast date $\tau_t$, and acquired context $(H_{Q,t},M_{Q,t-1})$,
with $q^H_{Q,t}\in\operatorname{int}\Delta^{K-1}$ and $K\geq 2$.
Let $S$ be a strictly proper scoring rule. Under the regularity conditions
of Appendix~\ref{app:grpo_centering}, draw $G>1$ independent reports from
the current policy, independently of the common outcome conditional on
$(Q,\tau_t,H_{Q,t},M_{Q,t-1})$. For this proposition, let
$R_{Q,t}^{(i)}=S(p_{Q,t}^{(i)},Y_Q)$, excluding format penalties.

\textnormal{(i)}
For the mean-centered advantage in Eq.~\ref{eq:no_std_advantage}, the
expected score-function update equals $\left(1-\frac{1}{G}\right)$ times
the corresponding uncentered update. Thus mean-centering preserves the
expected score-only reporting gradient direction.

\textnormal{(ii)}
Let $\sigma^{S}_{Q,t,k}$ denote the standard deviation of
$S(p_{Q,t},k)$ under the current policy's induced distribution over reports.
For fixed $\epsilon>0$, as $G\to\infty$, the expected standardized update
is locally equivalent, up to a positive factor, to the proper-score update
associated with
\begin{equation}
    \widetilde q^H_{Q,t}(k)
    =
    \frac{q^H_{Q,t}(k)/(\sigma^{S}_{Q,t,k}+\epsilon)}
         {\sum_{j=1}^{K}q^H_{Q,t}(j)/(\sigma^{S}_{Q,t,j}+\epsilon)}.
    \label{eq:tilted_target}
\end{equation}
Therefore $\widetilde q^H_{Q,t}=q^H_{Q,t}$ if and only if
$\sigma^{S}_{Q,t,1}=\cdots=\sigma^{S}_{Q,t,K}$; otherwise,
standard-deviation normalization induces a policy-dependent tilted
reporting target.
\end{proposition}

In this fixed-context, score-only analysis, the group mean preserves the expected reporting gradient, whereas standardization reweights outcome classes according to their reward scales in the large-group limit. This distortion need not always take the form of overconfidence; depending on the outcome-wise scales, it may sharpen the distribution, flatten it, or shift probability across classes. Appendix~\ref{app:grpo_centering} proves part~(i), and Appendix~\ref{app:std_properness} derives the multiclass effective target in part~(ii). This result motivates our normalization choice but characterizes only the local score-only reporting update, not calibration of the complete trained agent.
\vspace{-0.4em}
\paragraph{Group construction.}
For memory-free forecasting, we use the standard GRPO construction: each group contains $G$ rollouts for the same question and forecast date $(Q,\tau_t)$. Memory-on forecasting differs because the forecast at $\tau_t$ depends on the belief notebook accumulated at earlier dates. We therefore sample $G$ complete episodes for the same question and form the group at each date from the corresponding forecast steps across episodes. Each rollout carries its own notebook trajectory. We use the immediate reward at each forecast date without additional cross-date credit assignment. Thus, for memory-on episodes, the update is a per-date surrogate rather than the full episode-return gradient: earlier notebook actions may affect later forecasts but are not explicitly assigned those future rewards. Cross-date credit assignment remains future work.

\vspace{-0.4em}
\section{Experiments}
\label{sec:exp}
\vspace{-0.4em}
\subsection{Experimental Setup}
\label{sec:experimental_setup}
\paragraph{Training setup.}
We instantiate \ours with GLM-4.5-Air (106B-A12B)~\citep{glm45air},
whose public checkpoint was released on \texttt{07-28-2025}.
We train on the Forecast-Dojo training split~\citep{forecastdojo},
comprising 1,028 binary and multi-option events (3,965 event--date tasks) from Polymarket\footnote[1]{https://polymarket.us}, with forecast dates spanning \texttt{08-01-2025}--\texttt{02-28-2026}. Because all training events postdate the checkpoint release, their realized outcomes cannot be directly memorized from its pretraining data. We use a disjoint held-out validation set of events closing between \texttt{01-01-2026} and \texttt{02-28-2026} for checkpoint selection and hyperparameter tuning. We evaluate three variants: \emph{single-date}, which uses one forecast date for each event (Section \ref{sec:longitudinal}); \emph{memory-free}, which uses multi-date training without cross-date memory; and \emph{memory-on}, which carries
the rolling belief state across forecast dates. Following Forecast-Dojo's protocol, we curate a temporally indexed CC-News corpus spanning \texttt{03-2025} through \texttt{02-2026}, restricting retrieval at each forecast date to articles available by that date. We optimize negative Brier score together with the format reward defined in Section~\ref{sec:proper_objective}. Full implementation details are provided in Appendix~\ref{app:implementation}.
\begin{table*}[!htb]
\captionsetup[table]{skip=3pt}
\centering
\small
\setlength{\tabcolsep}{2.36pt}
\renewcommand{\arraystretch}{1.12}
\renewcommand{\modelcell}[1]{#1}
\definecolor{cigray}{gray}{0.52}
\RenewDocumentCommand{\metricci}{m >{\SplitArgument{1}{,}}m}{%
  \mainmetricci{#1}#2%
}
\makeatletter
\NewDocumentCommand{\mainmetricci}{m m m}{%
  \begin{tabular}[t]{@{}c@{}}
    #1\\[-4pt]
    \noalign{\global\let\CT@row@color\relax}%
    {\scriptsize\color{cigray}#2--#3}
  \end{tabular}%
}
\makeatother

\begin{threeparttable}
\caption{\textbf{Main results.}
Bold and underlined values mark the best and second-best point estimates,
respectively, within the GLM-4.5-Air block.
Gray ranges give 95\% confidence intervals
(event-clustered bootstrap for Forecast-Dojo; item bootstrap for FutureX).}
\label{tab:main}
\definecolor{FDLightBand}{HTML}{EFF3F6}
\begin{tabularx}{\textwidth}{
    >{\hsize=1.55\hsize\linewidth=\hsize\raggedright\arraybackslash}X
    *{5}{>{\hsize=0.89\hsize\linewidth=\hsize\centering\arraybackslash}X}
}

\toprule

&
\multicolumn{3}{c}{\textbf{Forecast-Dojo}}
&
\multicolumn{2}{c}{\textbf{FutureX}}
\\

\cmidrule(lr){2-4}
\cmidrule(lr){5-6}

\textbf{Model}
& Acc. (\%)\,$\uparrow$
& Brier\,$\downarrow$
& $-$Info-$\alpha$\,$\downarrow$
& Acc. (\%)\,$\uparrow$
& Brier\,$\downarrow$
\\

\midrule

Uniform Reference
& \metricci{22.9}{20.8, 25.1}
& \metricci{0.771}{0.749, 0.792}
& \metricci{0.727}{0.624, 0.830}
& \metricci{30.6}{28.7, 32.4}
& \metricci{0.680}{0.662, 0.699}
\\

\modelcell{Claude Sonnet 4.6$_\text{{\scriptsize\color{black}high}}$}
& \metricci{50.0}{44.5, 55.5}
& \metricci{0.624}{0.573, 0.675}
& \metricci{0.364}{0.253, 0.480}
& \metricci{58.9}{54.2, 63.6}
& \metricci{0.497}{0.444, 0.552}
\\

\modelcell{Kimi K2.5}
& \metricci{48.9}{43.8, 54.3}
& \metricci{0.663}{0.614, 0.710}
& \metricci{0.530}{0.414, 0.650}
& \metricci{57.5}{53.0, 61.9}
& \metricci{0.514}{0.464, 0.564}
\\

\modelcell{Qwen3.5-397B-A17B$_\text{{\scriptsize\color{black}think}}$}
& \metricci{48.2}{43.1, 53.5}
& \metricci{0.653}{0.606, 0.704}
& \metricci{0.482}{0.373, 0.608}
& \metricci{57.1}{52.5, 61.5}
& \metricci{0.556}{0.503, 0.607}
\\

\modelcell{Claude Sonnet 4.6$_\text{{\scriptsize\color{black}low}}$}
& \metricci{47.6}{41.6, 53.2}
& \metricci{0.656}{0.598, 0.714}
& \metricci{0.468}{0.346, 0.604}
& \metricci{54.2}{49.4, 59.1}
& \metricci{0.564}{0.506, 0.619}
\\

\modelcell{MiniMax-M2.5}
& \metricci{45.9}{40.9, 51.2}
& \metricci{0.719}{0.670, 0.768}
& \metricci{0.738}{0.608, 0.873}
& \metricci{53.3}{48.8, 57.7}
& \metricci{0.555}{0.505, 0.606}
\\

\midrule
\addlinespace[1pt]

\modelcell{GLM-4.5-Air}
& \metricci{41.6}{36.6, 46.9}
& \metricci{0.731}{0.685, 0.777}
& \metricci{0.726}{0.616, 0.877}
& \metricci{51.3}{47.0, 55.5}
& \metricci{0.612}{0.561, 0.664}
\\

\modelcell{\ \ +BLF}
& \metricci{45.8}{40.5, 51.4}
& \metricci{0.671}{0.630, 0.711}
& \metricci{0.458}{0.365, 0.555}
& \metricci{51.1}{45.9, 56.2}
& \metricci{0.592}{0.547, 0.639}
\\

\modelcell{\ \ +OpenForecaster}
& \metricci{42.8}{38.0, 48.1}
& \metricci{0.765}{0.719, 0.809}
& \metricci{0.845}{0.742, 0.957}
& \metricci{52.6}{48.5, 56.7}
& \metricci{0.583}{0.545, 0.622}
\\

\rowcolor{ours}
\modelcell{\ \ +\textbf{\ours}$_\text{{\scriptsize\color{black}single-date}}$}
& \metricci{47.2}{42.1, 52.3}
& \metricci{0.668}{0.619, 0.714}
& \metricci{0.507}{0.406, 0.619}
& \metricci{53.2}{48.7, 57.7}
& \metricci{\underline{0.540}}{0.494, 0.590}
\\

\rowcolor{ours}
\modelcell{\ \ +\textbf{\ours}$_\text{{\scriptsize\color{black}memory-on}}$}
& \metricci{\underline{49.3}}{43.7, 54.7}
& \metricci{\underline{0.630}}{0.577, 0.685}
& \metricci{\underline{0.413}}{0.313, 0.516}
& \metricci{\underline{55.2}}{50.4, 59.8}
& \metricci{0.545}{0.499, 0.592}
\\

\rowcolor{ours}
\modelcell{\ \ +\textbf{\ours}$_\text{{\scriptsize\color{black}memory-free}}$}
& \metricci{\textbf{50.5}}{45.3, 55.9}
& \metricci{\textbf{0.626}}{0.574, 0.677}
& \metricci{\textbf{0.399}}{0.304, 0.502}
& \metricci{\textbf{57.9}}{53.4, 62.4}
& \metricci{\textbf{0.528}}{0.483, 0.576}
\\

\bottomrule

\end{tabularx}
\end{threeparttable}
\vspace{-1.2em}
\end{table*}

\vspace{-0.3em}
\paragraph{Evaluation benchmarks.}
We evaluate on two forecasting benchmarks covering March through May 2026.
\emph{Forecast-Dojo} uses its held-out evaluation events from this period (230 events, 797 event--date pairs), while for \emph{FutureX}~\citep{futurex} we evaluate on L1--L2 questions (343 events), retaining its question-specific answer interface within our controlled offline-news setting. We keep the search infrastructure and temporally indexed corpus the same as training, extended through May 2026, with retrieval restricted to information available by each forecast date. Appendix~\ref{app:eval_data} provides further details on the characteristics and distribution of the evaluation data.

\vspace{-0.3em}
\paragraph{Baselines.}
We include a uniform reference that assigns probability $1/K$ to each of
the $K$ answer options, providing a non-informative baseline. We compare against the untrained GLM-4.5-Air base model, which isolates
the effect of forecasting-specific training. We additionally compare
against two forecasting-specific alternatives built on the same
backbone: BLF~\citep{blf}, a forecasting harness method, and OpenForecaster~\citep{openForecaster}, an RL-based forecasting training method. To contextualize performance beyond same-backbone comparisons, we also evaluate a set of larger and stronger off-the-shelf models---Claude Sonnet 4.6~\citep{anthropic2026claudesonnet46}, Kimi K2.5~\citep{kimi}, Qwen3.5-397B-A17B~\citep{qwen2026qwen35}, MiniMax-M2.5~\citep{minimax2026m25}---using the same agent loop and tool budget. We treat these models as references rather than controlled baselines. For every off-the-shelf model, either its officially reported knowledge cutoff or the checkpoint release date was before the evaluation-event period (Appendix \ref{app:knowledge_cutoff}).

\vspace{-0.6em}
\paragraph{Metrics.}
We report \textbf{Accuracy}: top-1 accuracy on Forecast-Dojo and the
benchmark's official answer-matching accuracy on FutureX
(Appendix~\ref{app:experimental_statistics}).
We also report the multiclass \textbf{Brier score}
\(
    \mathrm{Brier}(p,y)
    = \sum_{k=1}^{K}\bigl(p_k-\mathbbm{1}[k=y]\bigr)^2,
\)
where lower is better. We also report \textbf{Information Alpha}
(Info-$\alpha$) for Forecast-Dojo, the log-score gain over the crowd
forecast $p^{\mathrm{crowd}}$,
\(
    \mathrm{Info}\text{-}\alpha
    = \log p_y-\log p^{\mathrm{crowd}}_y.
\)
Positive values favor the model. Table~\ref{tab:main} reports
$-\mathrm{Info}\text{-}\alpha$ so that lower is better. For the GRPO ablation,
we additionally use the Murphy decomposition
\(
    \mathrm{Brier}
    = \mathrm{REL}
    - \mathrm{RES}
    + \mathrm{UNC}
\)
~\citep{murphy1973new} to distinguish reliability from resolution.
\vspace{-0.2em}

\begin{table*}[!htb]
\captionsetup[table]{skip=3pt}
\centering
\small
\setlength{\tabcolsep}{4pt}
\renewcommand{\arraystretch}{1.00}
\definecolor{cigray}{gray}{0.52}
\RenewDocumentCommand{\metricci}{m >{\SplitArgument{1}{,}}m}{%
  \gainmetricci{#1}#2%
}
\makeatletter
\NewDocumentCommand{\gainmetricci}{m m m}{%
  \begin{tabular}[t]{@{}c@{}}
    #1\\[-4pt]
    \noalign{\global\let\CT@row@color\relax}%
    {\scriptsize\color{cigray}#2--#3}
  \end{tabular}%
}
\makeatother
\begin{threeparttable}

\caption{\textbf{Performance gain analysis.} Entries are Brier scores ($\downarrow$). Gray ranges denote 95\% event-clustered bootstrap confidence intervals. $^{*}p<.05$, $^{**}p<.01$, and $^{***}p<.001$.}

\label{tab:gain_analysis}
\begin{tabularx}{\textwidth}{
    >{\hsize=1.30\hsize\linewidth=\hsize\raggedright\arraybackslash}X
    *{3}{>{\hsize=0.90\hsize\linewidth=\hsize\centering\arraybackslash}X}
    >{\hsize=1.00\hsize\linewidth=\hsize\centering\arraybackslash}X
}

\toprule

&
\multicolumn{3}{c}{\textbf{Provided Evidence}}
&
\\

\cmidrule(lr){2-4}

\textbf{Reader}
& $H_{\text{Base}}$
& $H_{\text{\ours}}$
& None
& $\boldsymbol{\Delta_{\mathrm{acq}}}\uparrow$
\\

\midrule

\modelcell{GPT-5.6 Terra$_\text{{\scriptsize\color{black}xhigh}}$}
& \metricci{0.623}{0.576, 0.671}
& \metricci{0.612}{0.566, 0.659}
& \metricci{0.725}{0.673, 0.778}
& \metricsolo{\sig{+0.011}{*}}
\\

\modelcell{Claude Sonnet 4.6$_\text{{\scriptsize\color{black}high}}$}
& \metricci{0.649}{0.600, 0.701}
& \metricci{0.638}{0.590, 0.688}
& \metricci{0.763}{0.709, 0.815}
& \metricsolo{\sig{+0.012}{*}}
\\

\midrule

\modelcell{GLM-4.5-Air}
& \metricci{0.702}{0.661, 0.743}
& \metricci{0.686}{0.644, 0.727}
& \metricci{0.801}{0.755, 0.847}
& \metricsolo{\sig{+0.016}{**}}
\\

\rowcolor{ours}
\modelcell{\ \ +\ours{}$_\text{{\scriptsize\color{black}memory-free}}$}
& \metricci{0.642}{0.592, 0.694}
& \metricci{0.628}{0.575, 0.681}
& \metricci{0.765}{0.713, 0.816}
& \metricsolo{\sig{+0.014}{*}}
\\

\midrule

\textbf{$\boldsymbol{\Delta_{\mathrm{report}}}\uparrow$}
& \sig{+0.060}{***}
& \sig{+0.058}{***}
& \sig{+0.035}{**}
&
\\

\bottomrule

\end{tabularx}

\end{threeparttable}
\vspace{-0.4em}
\end{table*}

\begin{figure*}[!htb]
    \centering
    \begin{minipage}[b]{0.49\textwidth}
        \centering
        \includegraphics[width=\linewidth]{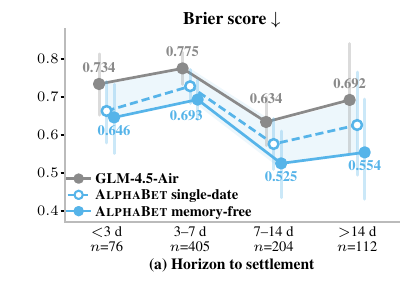}
        \phantomsubcaption
        \label{fig:horizon}
    \end{minipage}
    \hfill
    \begin{minipage}[b]{0.49\textwidth}
        \centering
        \includegraphics[width=\linewidth]{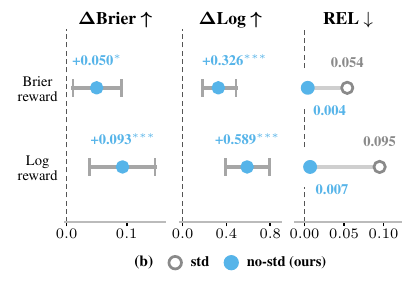}
        \phantomsubcaption
        \label{fig:grpo_summary}
    \end{minipage}
    \vspace{-2em}
    \caption{
    \textbf{(a)} Brier gain from longitudinal memory-free over single-date training across forecast horizons; positive values favor longitudinal training.
    \textbf{(b)} Final-checkpoint gains from removing GRPO standard-deviation normalization,
    $\Delta\mathrm{Brier}=\mathrm{Brier}_{\mathrm{std}}-\mathrm{Brier}_{\mathrm{no\text{-}std}}$ and
    $\Delta\mathrm{Log}=\mathrm{Log}_{\mathrm{no\text{-}std}}-\mathrm{Log}_{\mathrm{std}}$, so positive values favor no-std GRPO;
    the right strip shows Murphy reliability (REL) under each normalization.
    Bars are paired event-clustered 95\% CIs; $^{*}p<.05$, $^{***}p<.001$.
    }
    \label{fig:analysis}
    \vspace{-1em}
\end{figure*}

\subsection{Forecasting Performance}
\vspace{-0.4em}
\label{sec:main_results}

Table~\ref{tab:main} shows that \ours improves GLM-4.5-Air consistently across both benchmarks, including an 8.9-point accuracy gain and a Brier reduction from 0.731 to 0.626 on Forecast-Dojo. It also outperforms both the harness and RL training baselines. OpenForecaster gains less, likely because it is optimized on fixed retrieved context rather than through an agentic loop. At test time it searches and reads less than the base model (8.3 vs.\ 12.4 tool calls per forecast). Its reporting improvements thus operate on a thinner evidence base, reflecting the information gap in Proposition~\ref{prop:longitudinal_properness}.

\ours makes GLM-4.5-Air competitive with much stronger off-the-shelf forecasters, outperforming most open-weight baselines and approaching Claude Sonnet 4.6 (high) across both benchmarks. Interestingly, adding memory does not improve further: \ours{}$_{\text{memory-on}}$ is slightly worse than the memory-free variant on every metric, though still stronger than the single-date and base-model baselines on most metrics. A plausible explanation is that each forecast date is rewarded independently, so the notebook receives no direct credit for helping later forecasts. We leave explicit cross-date credit assignment to future work.

\vspace{-0.3em}
\subsection{Gain Analysis: Evidence Acquisition and Probability Reporting}
\label{sec:gain_analysis}
\vspace{-0.3em}
Proposition~\ref{prop:longitudinal_properness} separates forecasting quality into evidence acquisition and reporting, but an end-to-end agent entangles the two as it both gathers evidence and predicts from it. To disentangle them, we extract the article sets retrieved by the base model and by \ours, $H_{\text{Base}}$ and $H_{\ours}$, and ask tool-free readers to forecast from these frozen evidence sets under the same prompt and scoring protocol. With evidence fixed, swapping readers isolates reporting differences; with the reader fixed, swapping $H_{\text{Base}}$ and $H_{\ours}$ measures the downstream value of the acquired evidence. We repeat the latter comparison with GPT-5.6 Terra~\citep{openai2026gpt56} and Claude Sonnet 4.6 as independent readers to test whether $H_{\ours}$ is useful beyond \ours itself. We also include a question-only condition to control for each reader's default reporting bias in the absence of retrieved evidence. Appendix~\ref{app:gain_gap} gives the details of this analysis setup.
\begin{figure}[!htbp]
      \centering
      \includegraphics[width=\linewidth]{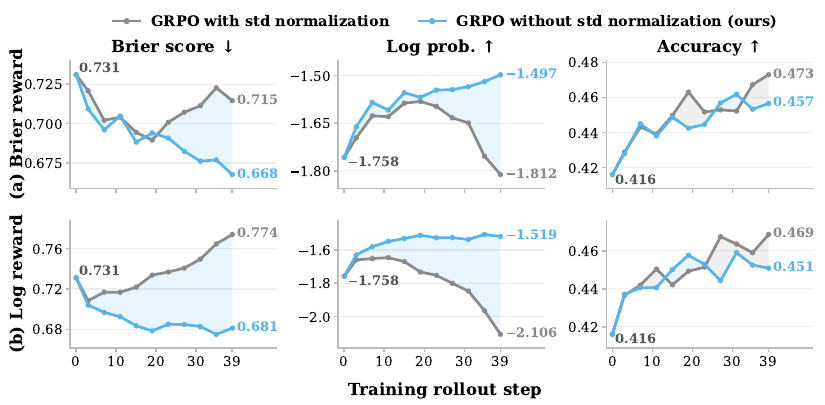}
      \vspace{-2em}
      \caption{Effect of GRPO reward standardization. Std normalization progressively degrades proper-score performance under both Brier and log-score rewards, while accuracy remains comparable.}
      \label{fig:grpo_std}
      \vspace{-1em}
\end{figure}

Table~\ref{tab:gain_analysis} shows gains from both evidence acquisition and reporting. Every reader, including the independently trained GPT-5.6 Terra and Claude Sonnet 4.6, performs better on $H_{\ours}$ than on $H_{\text{Base}}$ (0.011--0.016 Brier), indicating that \ours retrieves broadly more useful evidence rather than evidence tailored only to its own reader. Reporting also improves substantially: with the same evidence, \ours outperforms the base by 0.058--0.060 Brier. Its 0.035 advantage in the question-only condition accounts for part of this gap, while the additional gain with evidence suggests that \ours also extracts more value from what it reads.

\subsection{Longitudinal Supervision}
\label{sec:longitudinal}
\vspace{-0.3em}
To isolate the benefit of training across multiple dated states, we compare against \ours{}$_{\text{single-date}}$, which retains one fixed forecast date for each of the same 1,028 training events. We stratify events evenly across early, middle, and late relative forecast positions to avoid bias toward a particular horizon, yielding 1,028 training states versus 3,965 event--date states for longitudinal training. All other training settings and checkpoint selection are held fixed, giving both runs the same rollout/update budget at each step. Appendix~\ref{app:single_date} provides the full construction.

On Forecast-Dojo, \ours{}$_{\text{memory-free}}$ further improves over \ours{}$_{\text{single-date}}$, reducing Brier from 0.668 to 0.626 and raising accuracy from 47.2\% to 50.5\%. The same ordering holds on FutureX (Table~\ref{tab:main}). The longitudinal Brier gain is 0.072 beyond 14 days, 0.051 at 7--14 days, 0.035 at 3--7 days, and 0.017 within three days (Figure~\ref{fig:horizon}). The gain is positive across horizons and largest farther from settlement, suggesting that longitudinal training is especially helpful for earlier forecasts.

\subsection{Effect of GRPO Reward Standardization}
\label{sec:grpo_ablation}
\vspace{-0.4em}
Proposition~\ref{prop:grpo_std} shows that GRPO's within-group standard-deviation normalization can distort the reporting target induced by a proper scoring rule. We test this with controlled GLM-4.5-Air ablations under two proper scoring rules, Brier and log scores, differing only in whether reward standardization is applied. We re-evaluate every saved checkpoint on Forecast-Dojo and compare each std/no-std pair on their common-valid forecasts, isolating probability quality from output-validity effects. Appendix~\ref{app:grpo_std} provides the detailed setup and validity analysis.

Figure~\ref{fig:grpo_std} shows the same pattern under both objectives: standardization progressively worsens both proper scores, including the one not used for training, while top-1 accuracy remains broadly comparable. Final-checkpoint comparisons confirm significant proper-score gains from removing standardization, with no significant accuracy difference (Figure~\ref{fig:grpo_summary}). Murphy decomposition attributes the degradation primarily to reliability rather than resolution, indicating a distortion in probability reporting rather than outcome ranking. In our runs, this distortion manifests as excessive sharpening; additional diagnostics are provided in Appendix~\ref{app:grpo_std}.

\vspace{-0.6em}
\section{Related Work}
\label{sec:relate-work}
\vspace{-0.5em}

\paragraph{Event forecasting benchmarks and environments.}
ForecastQA and AutoCast study event prediction from time-restricted news~\citep{jin2021forecastqa,zou2022forecasting}.
ForecastBench and FutureX introduce live evaluations on unresolved questions, continuously incorporating newly arising events~\citep{forecastbench,futurex}.
More recent agentic environments make evidence gathering part of the forecasting process: MIRAI provides tools over historical events and news~\citep{ye2024mirai}, while FutureSim evaluates adaptation through chronological news replay~\citep{futuresim}.
Forecast-Dojo provides aligned forecast dates and temporally restricted research tools, and demonstrates supervised fine-tuning on agent trajectories~\citep{forecastdojo}.
Building on Forecast-Dojo, \ours{} turns these dated interactions into an outcome-based agentic RL training problem, using each event's eventual outcome to supervise decisions across multiple information states.
\vspace{-0.5em}
\paragraph{Learning to forecast with language models.}
Prior work trains forecasters through supervision on human forecasts or reasoning traces.
AutoCast++ combines news reranking and summarization with dated human forecasts~\citep{yan2024autocast}, while other approaches fine-tune on selected reasoning traces~\citep{approachHuman} or maintain explicit beliefs during iterative tool use~\citep{blf}.
Outcome-based RL instead learns from realized event outcomes~\citep{outcomeRL,turtel2026futureaslabel}, with OpenForecaster scaling this approach using news-derived questions and offline retrieval~\citep{openForecaster}.
However, existing outcome-based RL methods largely optimize isolated, single-date forecasts rather than the evolving research process of an agent.
In contrast, \ours{} introduces an agentic, longitudinal RL framework that jointly optimizes evidence acquisition and probability reporting across multiple dates of the same event, with proper-score feedback at each date.
\vspace{-0.5em}
\paragraph{Reward normalization and probabilistic calibration.}
Prior work studies probabilistic calibration through supervised fine-tuning~\citep{lin2022teaching}, reinforcement learning~\citep{band2024linguistic}, and proper-scoring objectives~\citep{gneiting2007strictly}.
However, using a proper scoring rule as a reward does not guarantee that the RL update preserves its elicitation incentives.
\citet{bereket2025uncalibratedreasoninggrpoinduces} show that GRPO's standard-deviation normalization~\citep{shao2024deepseekmath} can induce overconfidence for binary stochastic outcomes.
Removing this normalization has also been explored in Dr.~GRPO~\citep{liu2025understandingr1zeroliketrainingcritical} and forecasting RL~\citep{outcomeRL,openForecaster}.
\ours{} provides a theoretical analysis of how standard-deviation normalization shifts the probability-reporting target under multi-class proper scores, and empirically demonstrates the resulting degradation in probabilistic forecasting despite largely unchanged accuracy.
\section{Conclusion}
\label{sec:conclusion}
\vspace{-0.5em}
We introduced \ours{}, an outcome-based RL framework that trains a single LLM
agent to acquire evidence and report probabilities as information evolves.
Proper-score feedback allows each event's resolved outcome to supervise
forecasts across multiple pre-resolution dates. Our analysis decomposes forecasting regret into evidence-acquisition and probability-reporting gaps and explains how GRPO's standard-deviation normalization can distort the reporting target. Experiments on Forecast-Dojo and FutureX show consistent gains over single-date training, with controlled evidence--reader swaps identifying improvements in both capabilities.

\bibliography{ref}
\bibliographystyle{iclr2027_conference}

\appendix
\renewcommand{\thesection}{\Alph{section}}
\newpage

\begin{center}
	{\Large \textbf{Appendix for \ours}}
\end{center}

\startcontents[sections]
\printcontents[sections]{l}{1}{\setcounter{tocdepth}{2}}

\section{Theoretical Details}
\label{app:theory}

\subsection{Effect of the Format Reward}
\label{app:format_reward}

We prove the format-reward claim in
Section~\ref{sec:proper_objective}. Assume that the objectives are finite
and that there exists a forecasting-optimal policy $\pi_{\theta^\star}$
that passes every format check almost surely. Under this feasibility
assumption, the format reward preserves the optimal forecasting value;
the training optima are exactly the forecasting optima with zero
expected format penalty.

\begin{proof}
The coefficients in Eq.~\ref{eq:training_reward} are nonnegative, and
each format check takes values in $[0,1]$. Consequently,
$R_{Q,t}^{\mathrm{format}}\leq 0$ almost surely and
$J_{\mathrm{format}}(\theta)\leq 0$ for every $\theta$.
Since $\pi_{\theta^\star}$ passes all checks,
$J_{\mathrm{format}}(\theta^\star)=0$. Therefore, for every $\theta$,
\begin{equation}
\begin{aligned}
J_{\mathrm{train}}(\theta)
&\leq J_{\mathrm{forecast}}(\theta)
\leq J_{\mathrm{forecast}}(\theta^\star)
=J_{\mathrm{train}}(\theta^\star).
\end{aligned}
\label{eq:format_preserves_optimum}
\end{equation}
This proves
$\max_\theta J_{\mathrm{train}}(\theta)
=\max_\theta J_{\mathrm{forecast}}(\theta)$.

Now let $\widehat\theta$ be any maximizer of $J_{\mathrm{train}}$.
The two ends of Eq.~\ref{eq:format_preserves_optimum} are equal at
$\widehat\theta$, so both inequalities are equalities. Hence
$J_{\mathrm{forecast}}(\widehat\theta)
=J_{\mathrm{forecast}}(\theta^\star)$ and
$J_{\mathrm{format}}(\widehat\theta)=0$.
Conversely, any forecasting optimum with zero expected format penalty
attains this same optimal training value. Equivalently,
\[
\underset{\theta}{\arg\max}\,J_{\mathrm{train}}(\theta)
=
\left\{\theta\in
\underset{\vartheta}{\arg\max}\,J_{\mathrm{forecast}}(\vartheta)
:\ J_{\mathrm{format}}(\theta)=0\right\}.
\]
\end{proof}

This result preserves the forecasting optimum under the feasibility
assumption. It does not assert that the format term leaves gradients
unchanged or that every improvement in $J_{\mathrm{train}}$ improves
$J_{\mathrm{forecast}}$.

\subsection{Properness of the Agentic Forecasting Objective}
\label{app:properness}

We prove Proposition~\ref{prop:longitudinal_properness} for the targets
$q_{Q,t}$ and $q^H_{Q,t}$, discrepancy $D_S$, and optimal forecasting
value $J^\star$ defined in the main text.

\paragraph{Probability model and integrability.}
Expectations in this proof are population expectations over questions,
outcomes, dated evidence, and policy rollouts. The law of the question,
its scheduled dates, the corpus, and the outcome does not depend on the
forecasting policy. Thus $J^\star$ is policy-independent, although the
acquired context and its target may depend on the policy.
The training data provide an empirical estimate of this expectation;
the argument does not assume that rerunning a resolved question
resamples its outcome.
We make the finite-score requirement explicit by assuming, for every
policy under consideration,
\[
\mathbb E\!\left[
\sum_{t=1}^{T_Q}\left(
|S(q_{Q,t},Y_Q)|+|S(q^H_{Q,t},Y_Q)|+|S(p_{Q,t},Y_Q)|
\right)\right]<\infty.
\]
This ensures that the conditioning and summation steps below are valid.
For the logarithmic score, zero-probability outcomes contribute zero
to the corresponding conditional sums.

\paragraph{No-leakage conditions.}
We write $\tau_{Q,t}$ for the scheduled date $\tau_t$ of question $Q$.
The incoming notebook $M_{Q,t-1}$ is empty when memory is disabled.
The no-leakage assumption takes the following form, for every class
$k$ and forecast step, almost surely:
\begin{equation}
\begin{aligned}
\Pr(Y_Q=k\mid Q,\tau_{Q,t},\mathcal I_{\leq\tau_{Q,t}},H_{Q,t},M_{Q,t-1})
&=q_{Q,t}(k),\\
\Pr(Y_Q=k\mid Q,\tau_{Q,t},H_{Q,t},M_{Q,t-1},p_{Q,t})
&=q^H_{Q,t}(k).
\end{aligned}
\label{eq:no_leakage_conditions}
\end{equation}
The first equality requires both the current evidence and the incoming
notebook to carry no outcome information beyond the dated corpus.
The second requires the report to carry no additional outcome information
beyond these inputs. Reasoning, computation, and notebook updates may use both.
By the tower property and the first equality, 
\begin{equation}
\begin{aligned}
q^H_{Q,t}(k)
&=\Pr(Y_Q=k\mid Q,\tau_{Q,t},H_{Q,t},M_{Q,t-1})\\
&=\mathbb E\!\left[q_{Q,t}(k)\mid Q,\tau_{Q,t},H_{Q,t},M_{Q,t-1}\right].
\end{aligned}
\label{eq:context_target_relation}
\end{equation}

\begin{proof}[Proof of Proposition~\ref{prop:longitudinal_properness}]
Fix a policy. For the per-step calculation, fix a question $Q$ and a
scheduled step $t\leq T_Q$; expectations in the next three displays
are conditional on this question. We restore the outer question
expectation when summing the per-step identities.

\emph{Step 1: Split the score difference.}
Adding and subtracting $S(q^H_{Q,t},Y_Q)$ gives
\begin{equation}
\begin{aligned}
&\mathbb E[S(q_{Q,t},Y_Q)-S(p_{Q,t},Y_Q)]\\
&\quad=\mathbb E[S(q_{Q,t},Y_Q)-S(q^H_{Q,t},Y_Q)]
+\mathbb E[S(q^H_{Q,t},Y_Q)-S(p_{Q,t},Y_Q)].
\end{aligned}
\label{eq:score_gap_split}
\end{equation}

\emph{Step 2: Identify the two discrepancy terms.}
Conditional on
$(Q,\tau_{Q,t},\mathcal I_{\leq\tau_{Q,t}},H_{Q,t},M_{Q,t-1})$,
both $q_{Q,t}$ and $q^H_{Q,t}$ are fixed and the conditional law of
$Y_Q$ is $q_{Q,t}$. Hence, by iterated expectation,
\begin{equation}
\begin{aligned}
&\mathbb E[S(q_{Q,t},Y_Q)-S(q^H_{Q,t},Y_Q)]\\
&\quad=\mathbb E\!\left[
\sum_{k=1}^{K}q_{Q,t}(k)
\bigl(S(q_{Q,t},k)-S(q^H_{Q,t},k)\bigr)\right]\\
&\quad=\mathbb E[D_S(q_{Q,t},q^H_{Q,t})].
\end{aligned}
\label{eq:information_gap_identity}
\end{equation}
Likewise, conditional on $(Q,\tau_{Q,t},H_{Q,t},M_{Q,t-1},p_{Q,t})$,
both $q^H_{Q,t}$ and $p_{Q,t}$ are fixed and the conditional law of
$Y_Q$ is $q^H_{Q,t}$. Therefore,
\begin{equation}
\begin{aligned}
&\mathbb E[S(q^H_{Q,t},Y_Q)-S(p_{Q,t},Y_Q)]\\
&\quad=\mathbb E\!\left[
\sum_{k=1}^{K}q^H_{Q,t}(k)
\bigl(S(q^H_{Q,t},k)-S(p_{Q,t},k)\bigr)\right]\\
&\quad=\mathbb E[D_S(q^H_{Q,t},p_{Q,t})].
\end{aligned}
\label{eq:reporting_gap_identity}
\end{equation}

\emph{Step 3: Sum over the episode and characterize equality.}
Summing the preceding identities over $t\leq T_Q$ and averaging over
questions gives
\begin{equation}
\begin{aligned}
J^\star-J_{\mathrm{forecast}}(\theta)
=\mathbb E\!\left[\sum_{t=1}^{T_Q}
\left(D_S(q_{Q,t},q^H_{Q,t})
+D_S(q^H_{Q,t},p_{Q,t})\right)\right].
\end{aligned}
\label{eq:forecast_objective_gap}
\end{equation}
Since $J^\star$ is independent of $\theta$, maximizing
$J_{\mathrm{forecast}}$ is equivalent to minimizing the right-hand
side. This proves Eq.~\ref{eq:forecasting_gap_decomposition}.

Strict properness gives $D_S(q,p)\geq0$, with equality exactly when
$p=q$. The sum in Eq.~\ref{eq:forecast_objective_gap} is consequently
a nonnegative random variable. Its expectation is zero if and only
if the sum is zero almost surely, which, because every summand is
nonnegative, is equivalent to
\begin{equation}
q^H_{Q,t}=q_{Q,t},\qquad p_{Q,t}=q^H_{Q,t}
\quad\text{for every }t\leq T_Q\text{, almost surely}.
\label{eq:optimal_forecasting_conditions}
\end{equation}
Conversely, these equalities make every discrepancy zero. Thus a policy
attains the optimal forecasting value $J^\star$ if and only if
$p_{Q,t}=q^H_{Q,t}=q_{Q,t}$ at every scheduled step almost surely.
\end{proof}

\paragraph{Interpretation of the decomposition.}
Equation~\ref{eq:forecast_objective_gap} identifies the expected
\emph{sum} of the information and reporting gaps. An increase in
$J_{\mathrm{forecast}}$ decreases this sum, but need not decrease
both components separately or improve every date. The identity applies
to both memory-free and memory-on forecasting under the stated
no-leakage conditions. Together with Appendix~\ref{app:format_reward},
it also implies that every maximizer of $J_{\mathrm{train}}$ minimizes
this expected sum under the format-feasibility assumption.

\paragraph{Why one resolved outcome can supervise every forecast date.}
The same $Y_Q$ scores every step, but the proof uses conditional
expectation at each date and linearity across dates, not independent
labels across dates. Consequently, one resolved outcome can score
multiple dated forecasts without imposing the same target distribution
on them. Neither latent target needs to be observed to compute the
training score $S(p_{Q,t},Y_Q)$.

\subsection{Scope of the Properness Guarantee}
\label{app:properness_scope}

The strict-properness statements apply to the negative Brier score and
the ideal logarithmic score, subject to the integrability requirements
above. The negative Brier reward specified in the experimental setup
requires no probability floor. Indeed, expanding its expectation gives
\[
\mathbb E_{Y\sim q}[S_{\mathrm{Brier}}(p,Y)]
=-\|p\|_2^2+2\langle q,p\rangle-1,
\qquad
D_{S_{\mathrm{Brier}}}(q,p)=\|q-p\|_2^2.
\]
Thus the score is strictly proper on the full simplex, including its
boundary. The log-floor qualification below concerns the logarithmic
option, not negative-Brier training.

For the logarithmic option described in
Section~\ref{sec:proper_objective}, the ideal score is
$S_{\log}(p,Y)=\log p(Y)$. If the realized-outcome probability is
floored at $\varepsilon_{\mathrm{score}}=10^{-3}$, the score becomes
$\widetilde S_{\log}(p,Y)=\log\max\{p(Y),\varepsilon_{\mathrm{score}}\}$.
This floored score is not proper on the full simplex, as the following
counterexample shows.

\begin{proof}[Verification of the counterexample]
Let $q=(1-\delta,\delta)$, where
$0<\delta<\varepsilon_{\mathrm{score}}$, and compare truthful reporting
with the report $p^0=(1,0)$. Since
$1-\delta>\varepsilon_{\mathrm{score}}$,
\[
\begin{aligned}
\mathbb E_{Y\sim q}[\widetilde S_{\log}(p^0,Y)]
&=(1-\delta)\log 1+\delta\log\varepsilon_{\mathrm{score}},\\
\mathbb E_{Y\sim q}[\widetilde S_{\log}(q,Y)]
&=(1-\delta)\log(1-\delta)+\delta\log\varepsilon_{\mathrm{score}}.
\end{aligned}
\]
Subtracting yields
\begin{equation}
\mathbb E_{Y\sim q}\!\left[
\widetilde S_{\log}(p^0,Y)-\widetilde S_{\log}(q,Y)\right]
=-(1-\delta)\log(1-\delta)>0.
\label{eq:floored_log_counterexample}
\end{equation}
Thus the truthful report is not a maximizer of the floored score.
\end{proof}

The floored and ideal scores agree when $p(Y)\geq
\varepsilon_{\mathrm{score}}$, but this pointwise agreement does not
preserve the optimum on the full simplex. Accordingly,
Proposition~\ref{prop:longitudinal_properness} gives an exact
properness guarantee for the Brier score and the ideal log score,
not for the floored log score. The same qualification applies to the
strict-properness interpretation of
Proposition~\ref{prop:grpo_std}: the score-function identities
below remain valid for finite floored scores, but their interpretation
as uniquely eliciting a specified probability target requires a
strictly proper score.

\subsection{Mean-Centering Preserves the Proper-Score Update}
\label{app:grpo_centering}

We prove part~(i) of Proposition~\ref{prop:grpo_std}.
The setup and regularity conditions in this subsection also apply to
part~(ii). Fix $(Q,\tau_t,H_{Q,t},M_{Q,t-1})$ and its conditional outcome law
$q^H_{Q,t}$. This is a local reporting analysis: the context and its
outcome law are held fixed when differentiating the report policy.
We analyze the forecasting-score component, not the mixed reward
containing format penalties. Accordingly, the notation $R_{Q,t}^{(i)}$
in this subsection and Appendix~\ref{app:std_properness} is local to
this analysis and denotes $S(p_{Q,t}^{(i)},Y_Q)$, rather than the total
reward in Section~\ref{sec:proper_objective}. The format-feasibility
result does not remove format penalties from a gradient or a group
standard deviation.

\paragraph{Report policy and regularity.}
Let $\rho_{\vartheta,Q,t}(p\mid H_{Q,t},M_{Q,t-1})$ denote the induced probability
of report $p$, abbreviated as $\rho_\vartheta(p)$. Assume that, in a
neighborhood of the current parameter $\theta$, these probabilities
are continuously differentiable and positive on a common finite set
of reports $\mathcal P\subseteq\Delta^{K-1}$. The report vectors in
$\mathcal P$ are fixed, and $S(p,k)$ is finite for every
$p\in\mathcal P$ and class $k$. For the ideal log score, this requires
$p(k)>0$ on the report support.

The report-level score function is the logarithmic derivative of
this sampling distribution:
\begin{equation}
\psi_\theta(p)\triangleq\nabla_\theta\log\rho_\theta(p),
\qquad
\mathbb E_{p\sim\rho_\theta}[\psi_\theta(p)]
=\sum_{p\in\mathcal P}\nabla_\theta\rho_\theta(p)
=\nabla_\theta 1=0.
\label{eq:report_score_function}
\end{equation}
Define the local expected forecasting score, with $q^H_{Q,t}$ fixed,
by
\begin{equation}
\mathcal J_{Q,t}^{H}(\vartheta)
\triangleq\mathbb E_{p\sim\rho_\vartheta}\!\left[
\sum_{k=1}^{K}q^H_{Q,t}(k)S(p,k)\right].
\label{eq:local_reporting_objective}
\end{equation}
Differentiating the finite sum gives
\begin{equation}
\begin{aligned}
\nabla_\theta\mathcal J_{Q,t}^{H}(\theta)
&=\sum_{p\in\mathcal P}\nabla_\theta\rho_\theta(p)
\sum_{k=1}^{K}q^H_{Q,t}(k)S(p,k)\\
&=\mathbb E_{Y_Q\sim q^H_{Q,t},\,p\sim\rho_\theta}
[\psi_\theta(p)S(p,Y_Q)],
\end{aligned}
\label{eq:local_reporting_gradient}
\end{equation}
where the outcome and the sampled report are independent conditional
on the fixed context.

\paragraph{Group sampling and update.}
Draw $G>1$ reports independently from $\rho_\theta$, independently of
$Y_Q\sim q^H_{Q,t}$, and score them using this same outcome. Thus,
conditional on the fixed context, the joint sampling law is
$q^H_{Q,t}(k)\prod_{i=1}^{G}\rho_\theta(p_{Q,t}^{(i)})$.
For this proof, set
\begin{equation}
R_{Q,t}^{(i)}=S(p_{Q,t}^{(i)},Y_Q),
\qquad
\bar R_{Q,t}=\frac1G\sum_{j=1}^{G}R_{Q,t}^{(j)}.
\label{eq:grpo_group_rewards_appendix}
\end{equation}
All expectations in the remainder of this subsection average over
both the shared outcome and the sampled reports. Define the
on-policy score-function estimator
\begin{equation}
\widehat g_{\mathrm{mean}}
=\frac1G\sum_{i=1}^{G}\psi_\theta(p_{Q,t}^{(i)})
\bigl(R_{Q,t}^{(i)}-\bar R_{Q,t}\bigr).
\label{eq:mean_centered_update_appendix}
\end{equation}
The sampled rewards and group statistics are treated as constants
in this update; only the report-policy log probability is differentiated.

\begin{proof}[Proof of Proposition~\ref{prop:grpo_std}(i)]
For every $i$, Eq.~\ref{eq:local_reporting_gradient} implies
\[
\mathbb E[\psi_\theta(p_{Q,t}^{(i)})R_{Q,t}^{(i)}]
=\nabla_\theta\mathcal J_{Q,t}^{H}(\theta).
\]
For $i\neq j$, condition on the shared outcome before factoring the
expectation. Independence of the reports and
Eq.~\ref{eq:report_score_function} give
\begin{equation}
\begin{aligned}
\mathbb E[\psi_\theta(p_{Q,t}^{(i)})R_{Q,t}^{(j)}]
&=\sum_{k=1}^{K}q^H_{Q,t}(k)
\mathbb E[\psi_\theta(p_{Q,t}^{(i)})\mid Y_Q=k]
\mathbb E[R_{Q,t}^{(j)}\mid Y_Q=k]\\
&=0.
\end{aligned}
\label{eq:mean_centering_cross_terms}
\end{equation}
Expanding the group mean and separating the diagonal and off-diagonal
terms therefore yields
\begin{equation}
\begin{aligned}
\mathbb E[\widehat g_{\mathrm{mean}}]
&=\frac1G\sum_{i=1}^{G}
\mathbb E[\psi_\theta(p_{Q,t}^{(i)})R_{Q,t}^{(i)}]
-\frac1{G^2}\sum_{i=1}^{G}\sum_{j=1}^{G}
\mathbb E[\psi_\theta(p_{Q,t}^{(i)})R_{Q,t}^{(j)}]\\
&=\nabla_\theta\mathcal J_{Q,t}^{H}(\theta)
-\frac{G}{G^2}\nabla_\theta\mathcal J_{Q,t}^{H}(\theta)\\
&=\left(1-\frac1G\right)
\nabla_\theta\mathcal J_{Q,t}^{H}(\theta).
\end{aligned}
\label{eq:mean_centering_gradient}
\end{equation}
The expected uncentered update is
\[
\mathbb E\!\left[\frac1G\sum_{i=1}^{G}
\psi_\theta(p_{Q,t}^{(i)})R_{Q,t}^{(i)}\right]
=\nabla_\theta\mathcal J_{Q,t}^{H}(\theta).
\]
Therefore Eq.~\ref{eq:mean_centering_gradient} is exactly the
$1-1/G$ proportionality claimed in part~(i). Since $1-1/G>0$,
mean-centering preserves the expected proper-score update direction
for every finite $G>1$.
\end{proof}

\subsection{Standard-Deviation Normalization and the Tilted Target}
\label{app:std_properness}

We prove part~(ii) under the fixed-context setup and regularity
conditions of Appendix~\ref{app:grpo_centering}, with
$q^H_{Q,t}(k)>0$ for every class and fixed $\epsilon>0$.
All policy-dependent quantities below are evaluated at the current
parameter $\theta$.

\paragraph{Outcome-specific score scales and the sampled update.}
Define the mean and standard deviation of the score across reports
when outcome $k$ is realized:
\begin{equation}
\begin{aligned}
\mu_{Q,t,k}&\triangleq\mathbb E_{p\sim\rho_\theta}[S(p,k)],\\
\sigma^S_{Q,t,k}&\triangleq
\left(\mathbb E_{p\sim\rho_\theta}
[(S(p,k)-\mu_{Q,t,k})^2]\right)^{1/2}.
\end{aligned}
\label{eq:outcome_reward_scale}
\end{equation}
Thus $\sigma^S_{Q,t,k}$ measures score variability across reports,
not uncertainty about which outcome will occur. For a sampled group,
let
\begin{equation}
\widehat\sigma_{Q,t}^{(G)}
\triangleq\left[\frac1G\sum_{i=1}^{G}
(R_{Q,t}^{(i)}-\bar R_{Q,t})^2\right]^{1/2}
\label{eq:group_reward_std}
\end{equation}
and define the standardized score-function estimator
\begin{equation}
\widehat g_{\mathrm{std}}^{(G)}
=\frac1G\sum_{i=1}^{G}\psi_\theta(p_{Q,t}^{(i)})
\frac{R_{Q,t}^{(i)}-\bar R_{Q,t}}
{\widehat\sigma_{Q,t}^{(G)}+\epsilon}.
\label{eq:standardized_update_appendix}
\end{equation}
As before, the rewards and the group statistics are not differentiated.

\begin{proof}[Proof of Proposition~\ref{prop:grpo_std}(ii)]
\leavevmode\par
\smallskip

\noindent
\emph{Step 1: Obtain the conditional large-group limit.}
\par\nobreak
\noindent
Condition on $Y_Q=k$. The reports remain i.i.d. from $\rho_\theta$.
Using the identity
\[
    \bigl(\widehat\sigma_{Q,t}^{(G)}\bigr)^2
    =
    \frac{1}{G}\sum_{i=1}^{G}
    \bigl(R_{Q,t}^{(i)}\bigr)^2
    -
    \bar R_{Q,t}^{\,2},
\]
the strong law of large numbers gives
$\bar R_{Q,t}\to\mu_{Q,t,k}$ and
$\widehat\sigma_{Q,t}^{(G)}\to\sigma^{S}_{Q,t,k}$
almost surely as $G\to\infty$. It also yields
\begin{equation}
    \begin{aligned}
        \frac{1}{G}\sum_{i=1}^{G}
        \psi_\theta\bigl(p_{Q,t}^{(i)}\bigr)
        &\xrightarrow{\mathrm{a.s.}}
        \mathbb{E}_{p\sim\rho_\theta}
        \bigl[\psi_\theta(p)\bigr]
        = 0,
        \\[0.5em]
        \frac{1}{G}\sum_{i=1}^{G}
        \psi_\theta\bigl(p_{Q,t}^{(i)}\bigr)
        R_{Q,t}^{(i)}
        &\xrightarrow{\mathrm{a.s.}}
        \mathbb{E}_{p\sim\rho_\theta}
        \bigl[\psi_\theta(p)S(p,k)\bigr].
    \end{aligned}
    \label{eq:group_statistics_limit}
\end{equation}
Here the almost-sure limits follow from the strong law of large
numbers, and the equality to zero follows from the score-function
identity in Eq.~\ref{eq:report_score_function}.

The update can be written exactly as
\[
\widehat g_{\mathrm{std}}^{(G)}
=\frac{
G^{-1}\sum_i\psi_\theta(p_{Q,t}^{(i)})R_{Q,t}^{(i)}
-\bar R_{Q,t}\,G^{-1}\sum_i\psi_\theta(p_{Q,t}^{(i)})}
{\widehat\sigma_{Q,t}^{(G)}+\epsilon}.
\]
Since the limiting denominator is at least $\epsilon>0$,
\[
\widehat g_{\mathrm{std}}^{(G)}
\xrightarrow{\mathrm{a.s.}}
\frac{\mathbb E_{p\sim\rho_\theta}[\psi_\theta(p)S(p,k)]}
{\sigma^S_{Q,t,k}+\epsilon}
\quad\text{conditional on }Y_Q=k.
\]
No independence between the sample mean and sample standard deviation
is required.

\par\smallskip
\noindent
\emph{Step 2: Establish the expected update limit and outcome reweighting.}
\par\nobreak
\noindent
We first obtain a deterministic bound that is uniform in the group
size and the sampled reports. Under the regularity conditions in
Appendix~\ref{app:grpo_centering}, define
\[
    C_S\triangleq
    \max_{p\in\mathcal P,\,1\leq k\leq K}|S(p,k)|,
    \qquad
    C_\psi\triangleq
    \max_{p\in\mathcal P}\|\psi_\theta(p)\|,
\]
where $\|\cdot\|$ denotes the Euclidean norm.
The finite-support and finite-score assumptions imply $C_S<\infty$.
Moreover, positivity and differentiability of $\rho_\theta(p)$ give
\[
    \|\psi_\theta(p)\|
    =
    \frac{\|\nabla_\theta\rho_\theta(p)\|}{\rho_\theta(p)}
    <\infty
    \qquad\text{for every }p\in\mathcal P.
\]
Since $\mathcal P$ is finite, $C_\psi<\infty$ as well.
These constants may depend on the fixed context and current parameter
$\theta$, but not on $G$ or the sampled group.

For every group member, the definition of $C_S$ gives
$|R_{Q,t}^{(i)}|=|S(p_{Q,t}^{(i)},Y_Q)|\leq C_S$.
Consequently, the sample mean satisfies
\[
    |\bar R_{Q,t}|
    =
    \left|\frac{1}{G}\sum_{j=1}^{G}R_{Q,t}^{(j)}\right|
    \leq
    \frac{1}{G}\sum_{j=1}^{G}|R_{Q,t}^{(j)}|
    \leq C_S,
\]
and each centered reward is bounded by
\[
    |R_{Q,t}^{(i)}-\bar R_{Q,t}|
    \leq |R_{Q,t}^{(i)}|+|\bar R_{Q,t}|
    \leq 2C_S.
\]
Also, $\widehat\sigma_{Q,t}^{(G)}\geq0$, so the denominator is at
least the fixed positive constant $\epsilon$.
Applying the triangle inequality and these bounds yields, for every
$G$ and every sampled group,
\begin{equation}
    \begin{aligned}
        \bigl\|\widehat g_{\mathrm{std}}^{(G)}\bigr\|
        &=
        \left\|
        \frac{1}{G}\sum_{i=1}^{G}
        \psi_\theta\bigl(p_{Q,t}^{(i)}\bigr)
        \frac{R_{Q,t}^{(i)}-\bar R_{Q,t}}
             {\widehat\sigma_{Q,t}^{(G)}+\epsilon}
        \right\|
        \\[0.4em]
        &\leq
        \frac{1}{G}\sum_{i=1}^{G}
        \frac{
            \bigl\|\psi_\theta\bigl(p_{Q,t}^{(i)}\bigr)\bigr\|
            \,|R_{Q,t}^{(i)}-\bar R_{Q,t}|
        }{\widehat\sigma_{Q,t}^{(G)}+\epsilon}
        \\[0.4em]
        &\leq
        \frac{1}{G\epsilon}\sum_{i=1}^{G}
        \bigl\|\psi_\theta\bigl(p_{Q,t}^{(i)}\bigr)\bigr\|
        \,|R_{Q,t}^{(i)}-\bar R_{Q,t}|
        \\[0.4em]
        &\leq
        \frac{1}{G\epsilon}\sum_{i=1}^{G}C_\psi(2C_S)
        \\[0.4em]
        &=\frac{2C_\psi C_S}{\epsilon}.
    \end{aligned}
    \label{eq:standardized_update_domination}
\end{equation}
The final bound is finite and independent of $G$ and of both the
shared outcome and the sampled reports. Together with the
conditional almost-sure limit from Step~1, it permits dominated
convergence componentwise, giving, for each class $k$,
\[
    \lim_{G\to\infty}
    \mathbb E\!\left[
        \widehat g_{\mathrm{std}}^{(G)}\mid Y_Q=k
    \right]
    =
    \frac{
        \mathbb E_{p\sim\rho_\theta}
        [\psi_\theta(p)S(p,k)]
    }{\sigma^S_{Q,t,k}+\epsilon}.
\]
Finally, average over the shared outcome using its conditional law
$q^H_{Q,t}$. Since the number of classes is finite,
\begin{equation}
    \begin{aligned}
        \lim_{G\to\infty}
        \mathbb E[\widehat g_{\mathrm{std}}^{(G)}]
        &=
        \sum_{k=1}^{K}
        \frac{q^H_{Q,t}(k)}{\sigma^S_{Q,t,k}+\epsilon}
        \mathbb E_{p\sim\rho_\theta}
        [\psi_\theta(p)S(p,k)]
        \\[0.4em]
        &=
        \mathbb E_{p\sim\rho_\theta}\!\left[
            \psi_\theta(p)\sum_{k=1}^{K}
            \frac{q^H_{Q,t}(k)}{\sigma^S_{Q,t,k}+\epsilon}
            \bigl(S(p,k)-\mu_{Q,t,k}\bigr)
        \right].
    \end{aligned}
    \label{eq:standardized_update_limit}
\end{equation}
The second equality uses
$\mathbb E_{p\sim\rho_\theta}[\psi_\theta(p)]=0$:
for each fixed $k$, the baseline $\mu_{Q,t,k}$ is independent of the
sampled report $p$, so its score-function expectation vanishes.

\emph{Step 3: Identify the proper-score objective and its tilted reporting target.}
Define
\begin{equation}
Z_{Q,t}\triangleq\sum_{k=1}^{K}
\frac{q^H_{Q,t}(k)}{\sigma^S_{Q,t,k}+\epsilon},
\qquad
\widetilde q^H_{Q,t}(k)\triangleq
\frac{q^H_{Q,t}(k)/(\sigma^S_{Q,t,k}+\epsilon)}{Z_{Q,t}}.
\label{eq:tilted_target_appendix}
\end{equation}
The normalizer $Z_{Q,t}$ is finite and positive, and
$\widetilde q^H_{Q,t}$ is a probability distribution.
In the calculation below, both quantities are evaluated at the
current parameter $\theta$ and held fixed; differentiation with
respect to $\vartheta$ acts only on the report probabilities
$\rho_\vartheta(p)$.
Substituting Eq.~\ref{eq:tilted_target_appendix} into the first
line of Eq.~\ref{eq:standardized_update_limit} gives
\begin{equation}
    \begin{aligned}
        &\lim_{G\to\infty}
        \mathbb E[\widehat g_{\mathrm{std}}^{(G)}]
        \\[0.3em]
        &\quad=
        Z_{Q,t}\,
        \mathbb E_{p\sim\rho_\theta}\!\left[
            \psi_\theta(p)
            \sum_{k=1}^{K}
            \widetilde q^H_{Q,t}(k)S(p,k)
        \right]
        \\[0.5em]
        &\quad\overset{(a)}{=}
        Z_{Q,t}
        \sum_{p\in\mathcal P}
        \rho_\theta(p)\psi_\theta(p)
        \sum_{k=1}^{K}
        \widetilde q^H_{Q,t}(k)S(p,k)
        \\[0.5em]
        &\quad\overset{(b)}{=}
        Z_{Q,t}
        \sum_{p\in\mathcal P}
        \bigl(\nabla_\theta\rho_\theta(p)\bigr)
        \sum_{k=1}^{K}
        \widetilde q^H_{Q,t}(k)S(p,k)
        \\[0.5em]
        &\quad\overset{(c)}{=}
        Z_{Q,t}
        \left.
        \nabla_\vartheta
        \left[
            \sum_{p\in\mathcal P}
            \rho_\vartheta(p)
            \sum_{k=1}^{K}
            \widetilde q^H_{Q,t}(k)S(p,k)
        \right]
        \right|_{\vartheta=\theta}
        \\[0.5em]
        &\quad\overset{(d)}{=}
        Z_{Q,t}
        \left.
        \nabla_\vartheta
        \mathbb E_{p\sim\rho_\vartheta}\!\left[
            \sum_{k=1}^{K}
            \widetilde q^H_{Q,t}(k)S(p,k)
        \right]
        \right|_{\vartheta=\theta}.
    \end{aligned}
    \label{eq:tilted_target_gradient}
\end{equation}
Here $(a)$ expands the expectation over the finite report support;
$(b)$ uses
$\rho_\theta(p)\psi_\theta(p)=\nabla_\theta\rho_\theta(p)$;
$(c)$ interchanges differentiation and finite summation, with
the reports, scores, and target weights held fixed;
and $(d)$ rewrites the finite sum as an expectation.

Thus, the limiting expected update is a positive multiple of the
gradient of a frozen-target proper-score objective.
Strict properness makes $\widetilde q^H_{Q,t}$ its unique optimal
report on the full simplex, which differs from the acquired-context
target whenever $\widetilde q^H_{Q,t}\neq q^H_{Q,t}$.
This identifies the local reporting incentive, not attainability
by the restricted report policy or convergence of training.

\emph{Step 4: Characterize when standardization preserves the original reporting target.}
Since $q^H_{Q,t}(k)>0$ for every $k$,
\[
\begin{aligned}
\widetilde q^H_{Q,t}=q^H_{Q,t}
&\ \Longleftrightarrow\ 
\frac{1}{Z_{Q,t}(\sigma^S_{Q,t,k}+\epsilon)}=1
\quad\text{for every }k\\
&\ \Longleftrightarrow\ 
\sigma^S_{Q,t,k}+\epsilon=Z_{Q,t}^{-1}
\quad\text{for every }k.
\end{aligned}
\]
This requires all outcome-specific score scales to be equal.
Conversely, if their common value is $s$, then
$Z_{Q,t}=1/(s+\epsilon)$ and the effective target equals $q^H_{Q,t}$.
Therefore,
\begin{equation}
    \widetilde q^H_{Q,t}=q^H_{Q,t}
    \quad\Longleftrightarrow\quad
    \sigma^S_{Q,t,1}=\cdots=\sigma^S_{Q,t,K}.
    \label{eq:tilted_target_equality_condition}
\end{equation}
Thus, under the full-support assumption on $q^H_{Q,t}$,
standardization preserves the original reporting target in the
fixed-policy large-group limit if and only if all outcome-specific
score standard deviations are equal.
When they are equal, the limiting expected update is a common
positive rescaling of the original proper-score gradient.
Otherwise, $\widetilde q^H_{Q,t}\neq q^H_{Q,t}$, so the locally
elicited report differs from the acquired-context target.
This establishes target distortion; whether that distortion
sharpens the forecast is addressed separately below.
This completes the proof.

\end{proof}

\paragraph{Why standardization loses outcome-wise scale.}
The scale loss also has a finite-group interpretation. For constants
$a_k>0$ and $b_k\in\mathbb R$, define
\begin{equation}
S^\star(p,k)=a_kS(p,k)+b_k.
\label{eq:outcome_score_rescaling}
\end{equation}
Within a group whose shared outcome is $k$, write $R_i$, $\bar R$,
and $\widehat\sigma$ for its rewards, mean, and standard deviation;
use a superscript $\star$ for the corresponding transformed quantities.
Then $R_i^\star-\bar R^\star=a_k(R_i-\bar R)$ and
$\widehat\sigma^\star=a_k\widehat\sigma$. With $\epsilon=0$,
on any nondegenerate group with $\widehat\sigma>0$,
\begin{equation}
\frac{R_i^\star-\bar R^\star}{\widehat\sigma^\star}
=\frac{R_i-\bar R}{\widehat\sigma}.
\label{eq:scale_blindness}
\end{equation}
Yet the expected transformed score assigns weight
$q^H_{Q,t}(k)a_k$ to $S(p,k)$; unequal $a_k$ can therefore change
the reporting target while leaving these normalized advantages
unchanged. With $\epsilon>0$ the exact relation becomes
\[
\frac{R_i^\star-\bar R^\star}{\widehat\sigma^\star+\epsilon}
=\frac{R_i-\bar R}{\widehat\sigma+\epsilon/a_k},
\]
so exact scale invariance no longer holds. The positive-$\epsilon$
reporting target is characterized by the proof above, rather than
by this invariance observation alone.

\paragraph{When does the tilt sharpen the forecast?}
Target distortion alone does not imply overconfidence. A sufficient
condition for sharpening is the opposite ordering of outcome
probabilities and score scales. Suppose that, after relabeling the
classes,
\begin{equation}
q^H_{Q,t}(1)\geq\cdots\geq q^H_{Q,t}(K),
\qquad
\sigma^S_{Q,t,1}\leq\cdots\leq\sigma^S_{Q,t,K}.
\label{eq:scale_ordering}
\end{equation}
Under this condition, the effective target does not decrease the probability
ratio of any higher-ranked class to any lower-ranked class compared with
$q^H_{Q,t}$, and it majorizes this distribution.

\begin{proof}[Proof of the sharpening claim]
Set $w_k=(\sigma^S_{Q,t,k}+\epsilon)^{-1}$, the inverse scale of
outcome $k$. Then $w_1\geq\cdots\geq w_K>0$ and
$\widetilde q^H_{Q,t}(k)=q^H_{Q,t}(k)w_k/Z_{Q,t}$.
Both $q^H_{Q,t}(k)$ and $w_k$ are nonincreasing, so their product,
and hence $\widetilde q^H_{Q,t}(k)$, is also nonincreasing.
For $i<j$, positivity of the outcome probabilities gives
\begin{equation}
\frac{\widetilde q^H_{Q,t}(i)}{\widetilde q^H_{Q,t}(j)}
=\frac{q^H_{Q,t}(i)}{q^H_{Q,t}(j)}\frac{w_i}{w_j}
\geq\frac{q^H_{Q,t}(i)}{q^H_{Q,t}(j)}.
\label{eq:tilted_odds_sharpen}
\end{equation}
For every $1\leq m<K$, using $\sum_k q^H_{Q,t}(k)=1$ gives
\begin{equation}
\begin{aligned}
&Z_{Q,t}\left(\sum_{i=1}^{m}\widetilde q^H_{Q,t}(i)
-\sum_{i=1}^{m}q^H_{Q,t}(i)\right)\\
&\quad=\sum_{i=1}^{m}q^H_{Q,t}(i)w_i
-\left(\sum_{i=1}^{m}q^H_{Q,t}(i)\right)
\left(\sum_{k=1}^{K}q^H_{Q,t}(k)w_k\right)\\
&\quad=\left(\sum_{i=1}^{m}q^H_{Q,t}(i)w_i\right)
\left(\sum_{j=m+1}^{K}q^H_{Q,t}(j)\right)
-\left(\sum_{i=1}^{m}q^H_{Q,t}(i)\right)
\left(\sum_{j=m+1}^{K}q^H_{Q,t}(j)w_j\right)\\
&\quad=\sum_{i\leq m}\sum_{j>m}
q^H_{Q,t}(i)q^H_{Q,t}(j)(w_i-w_j)\geq0.
\end{aligned}
\label{eq:tilted_majorization}
\end{equation}
The last inequality holds because $w_i\geq w_j$ for $i\leq m<j$.
Since both vectors are sorted in the same order and both sum to one,
these prefix-sum inequalities establish that
$\widetilde q^H_{Q,t}$ majorizes $q^H_{Q,t}$.
\end{proof}

Without Eq.~\ref{eq:scale_ordering}, the tilt can instead flatten the
distribution or shift mass across classes. The proposition therefore
establishes target distortion, not universal overconfidence.

\paragraph{Relation to the binary case and scope.}
For $K=2$, Eq.~\ref{eq:tilted_target_appendix} has the binary
outcome-reweighting form discussed by
\citet{bereket2025uncalibratedreasoninggrpoinduces}.
Equation~\ref{eq:mean_centering_gradient} is exact for finite $G>1$;
Eq.~\ref{eq:tilted_target_gradient} is a fixed-policy, large-group
limit with fixed $\epsilon>0$. Its frozen weights change with the
policy, so it does not identify a single objective optimized
throughout training. For fixed policy scales, multiplying numerator
and denominator in Eq.~\ref{eq:tilted_target_appendix} by $\epsilon$
shows that $\widetilde q^H_{Q,t}\to q^H_{Q,t}$ as
$\epsilon\to\infty$.

\paragraph{Relation to the implemented rollout groups.}
Section~\ref{sec:calibration_grpo} groups independent rollouts by the
same $(Q,\tau_t)$ and, in memory-on forecasting, groups the
corresponding dated steps from $G$ complete episodes. A shared question
and date do not imply a shared acquired context: different searches can
produce different $H_{Q,t}$ even in independent forecasting, and
sequential rollouts additionally carry their own notebook histories.
The proof above fixes both $H_{Q,t}$ and $M_{Q,t-1}$ and isolates the
reporting effect of normalization at that fixed information state;
it is not an exact analysis of these heterogeneous-context groups.

The implemented sequential updates use per-date rewards without
additional cross-date credit assignment. Proposition~\ref{prop:longitudinal_properness}
characterizes the longitudinal objective, whereas
Proposition~\ref{prop:grpo_std} concerns the on-policy
score-function estimators displayed above. Neither result establishes
that the implemented per-date updates equal the full episode-return
gradient or proves that notebook-writing actions receive credit for
future forecasts. A clipped multi-update surrogate and cross-date
credit assignment remain outside this calculation.

\section{Implementation Details}
\label{app:implementation}

\ours{} is built on Forecast-Dojo's~\citep{forecastdojo} forecasting tasks and offline environment. We use its released task data\footnote[1]{https://huggingface.co/datasets/liqinye/Forecast-Dojo} and follow its protocol to curate our information corpus. Below, we describe how its dated interactions support policy training. Appendix~\ref{app:baseline_implementation} describes how the forecasting baselines are adapted to our setting.

\subsection{Forecasting Tasks}
\label{app:training_tasks}

\paragraph{Task source and input structure.}
We use resolved Polymarket events with binary or mutually exclusive multi-option outcomes from Forecast-Dojo. Each event supplies a question, resolution criterion, outcome label, and schedule of forecast dates. We retain these task specifications and select events with training forecast dates spanning August 2025 through February 2026. The inference unit depends on the memory setting. In the memory-free setting, each question--date pair is treated as an independent inference sample, so forecasts for the same event at different dates are generated separately with no state carried across dates. In the memory-on setting, each event is treated as a single inference sample whose rollout spans all of its forecast dates in chronological order. The agent therefore produces multiple forecasts within one episode while carrying its notebook state from one date to the next. Apart from this difference in input structure and state persistence, both settings use the same questions and date-specific evidence
boundaries.

\paragraph{Single-date variant.}
\label{app:single_date}
The single-date comparison retains one fixed forecast date per training event. We shuffle events with seed \texttt{20260918} and divide them evenly among early, middle, and late groups. Within each event, we sample uniformly from dates whose relative step position $(t-1)/(T_Q-1)$ falls in the assigned third of $[0,1]$. The selected dates remain fixed throughout training. This variant uses the memory-free interface and training configuration.

\subsection{Temporal Information Access}
\label{app:information_access}

\paragraph{Offline corpus.}
We curate a CC-News~\citep{nagel2016ccnews} corpus spanning March 2025 through May 2026, following Forecast-Dojo's protocol for English-language filtering, document cleaning, duplicate suppression, and timestamp assignment. Article dates come from structured publication or modification metadata, with crawl time as a fallback. Search and Read return stored snapshots, so the agent's research does not depend on changes to live webpages.

\paragraph{Dated retrieval.}
Articles are embedded with Qwen3-Embedding-8B~\citep{qwen3embedding} and searched using FAISS~\citep{faiss}. At forecast date $\tau_t$, retrieval is restricted to articles whose assigned UTC calendar date is no later than $\tau_t$. Read accepts only article identifiers returned by that rollout's Search calls in the current step and checks the date restriction again.
Advancing to the next scheduled date expands the searchable corpus.

\subsection{Agent Interface}
\label{app:agent_interface}

\begin{algorithm}[t]
\caption{RL training for \ours}
\label{alg:rollout_groups}
\small
\begin{algorithmic}
\REQUIRE Policy $\pi_\theta$; group size $G$; corpus $\mathcal I$; questions $Q$ with forecast dates $\tau_1<\dots<\tau_{T_Q}$; outcomes $Y_Q$
\WHILE{training budget remains}
  \STATE $\theta_{\mathrm{old}}\gets\theta$;\; sample inputs: $(Q,t)$ pairs if memory-free, questions $Q$ if memory-on
  \STATE \phase{Rollout} \sidenote{$G$ independent rollouts per input}
  \FOR{each sampled input and each $i\in[G]$}
    \STATE $M^{(i)}\gets\varnothing$ \sidenote{notebook; omitted when memory is off}
    \FOR{the sampled step $t$ if memory-free, otherwise $t=1,\dots,T_Q$}
      \STATE Fresh conversation with $(Q,\tau_t,M^{(i)})$;\; retrieval restricted to $\mathcal I_{\le\tau_t}$;\; tool budget reset
      \STATE $z^{(i)}_{Q,t},\,p^{(i)}_{Q,t},\,M^{(i)}\gets\pi_{\theta_{\mathrm{old}}}$ \sidenote{reason, use tools, forecast, update notebook}
    \ENDFOR
  \ENDFOR
  \STATE \phase{Reward} \sidenote{rollouts compete only with peers at the same date}
  \FORALL{$(Q,\tau_t)$ and $i\in[G]$}
    \STATE $R^{(i)}_{Q,t}\gets S\big(p^{(i)}_{Q,t},Y_Q\big)+R^{\mathrm{fmt},(i)}_{Q,t}$
    \STATE $A^{(i)}_{Q,t}\gets R^{(i)}_{Q,t}-\frac1G\sum_{j=1}^{G}R^{(j)}_{Q,t}$ \sidenote{no std.\ normalization}
  \ENDFOR
  \STATE \phase{Update}
  \STATE $\theta\gets$ clipped GRPO step on $\{z^{(i)}_{Q,t}\}$ with advantages $A^{(i)}_{Q,t}$ \sidenote{loss on generated tokens only}
\ENDWHILE
\RETURN $\pi_\theta$
\end{algorithmic}
\end{algorithm}

\paragraph{Prompt and research.}
Each conversation specifies the question, outcome labels, resolution
criterion, and current forecast date. The system prompt (shown in Appendix~\ref{app:prompts}) guides the agent from a base-rate estimate through targeted search, read, and compute to a probability forecast. Search returns article metadata (ID, title, URL, publication date, and retrieval score) and a short content snippet, with top-5 results by default
and an agent-adjustable result count. Read returns the selected article's text, truncated to 8,000 tokens. Compute executes Python in a fresh sandbox with no network or local file system access. The agent can use these observations to revise its reasoning or continue researching before committing its forecast.

\paragraph{Outputs and memory.}
Each step ends with an \texttt{<answer>} block containing a JSON mapping from
the supplied outcome labels to probabilities.
Memory-on rollouts also produce a \texttt{<belief\_notebook>} containing a
probability assessment, supporting evidence, and open questions.
The first step starts with an empty notebook. Each later step receives only
the preceding notebook, with a fresh conversation and tool budget.
The resolved outcome and all reward feedback remain outside the agent's
context throughout rollout generation.

\paragraph{Interaction limits.}
Each forecast date allows up to 80 tool rounds and 200 individual tool calls. The generation limit of 16,384 tokens applies to each model call, rather than to the entire multi-turn forecast step. Python calls have a 10-second time limit and a 4\,GiB memory limit. Table~\ref{tab:implementation_settings} demonstrates these limits alongside the optimization settings.

\subsection{Reward Implementation}
\label{app:reward_implementation}

\paragraph{Answer and notebook checks.}
The answer quality $f_t^{\mathrm{ans}}\in[0,1]$ is built from simple validity checks. Let $a_t$ be the mean of three binary checks on \texttt{<answer>}: the content is strict JSON, all probability values are numeric, and their sum after nonnegative clipping but before normalization lies in $[0.95,1.05]$. When memory is enabled, let $b_t$ be the mean of three notebook checks: \texttt{<belief\_notebook>} is parseable, it contains numeric probabilities in \texttt{assessment.p}, and those probabilities agree with the answer distribution. Agreement requires the same outcome labels and an absolute difference of at most $0.05$ for each probability. We set
\begin{equation}
    f_t^{\mathrm{ans}}=
    \begin{cases}
        a_t, & \text{memory-free forecasting},\\
        \tfrac12(a_t+b_t), & \text{memory-on forecasting}.
    \end{cases}
    \label{eq:answer_format_quality}
\end{equation}
The notebook is checked here only for structure and consistency.

\paragraph{Tool-call checks.}
The tool quality $f_t^{\mathrm{tool}}\in[0,1]$ is the fraction of calls whose tool names are registered by the environment, and is set to one when there are no calls. Thus, any sequence containing only supported tool names incurs no tool-format penalty. This term enforces a valid interface rather than judging whether the retrieved evidence is useful. We use $\lambda_{\mathrm{ans}}=1.0$ and $\lambda_{\mathrm{tool}}=0.5$ in Eq.~\ref{eq:training_reward}.

\begin{table}[!htb]
    \centering
    \small
    \renewcommand{\arraystretch}{1.0}
    \caption{GLM-4.5-Air training configurations. Paired values are memory-free | memory-on.}
    \begin{tabular}{@{}p{0.60\linewidth}p{0.36\linewidth}@{}}
        \toprule
        \textbf{Setting} & \textbf{Value} \\
        \midrule
        \multicolumn{2}{@{}l}{\textit{Rollout generation}} \\
        Rollouts per input, $G$ & 16 \\
        Sampling temperature / top-$p$ & 1.0 / 1.0 \\
        Prompt limit & 8,192 tokens \\
        Generation limit per model call & 16,384 tokens \\
        Tool budget per forecast date & 80 rounds; 200 calls \\
        Python budget per call & 10 s; 4\,GiB \\
        Rollout engines (GPUs per engine) & 16 (8) \\
        Rollout data-parallel attention & 8 \\
        \midrule
        \multicolumn{2}{@{}l}{\textit{Reward and advantage}} \\
        Proper-score reward & $-$Brier \\
        Answer-format penalty coefficient & 1.0 \\
        Tool-format penalty coefficient & 0.5 \\
        Advantage estimator & GRPO \\
        Reward standardization & Disabled, except in the ablation \\
        \midrule
        \multicolumn{2}{@{}l}{\textit{Policy optimization}} \\
        Optimizer & Adam; $(\beta_1,\beta_2,\epsilon)=(0.9,0.95,10^{-8})$ \\
        Learning rate & $10^{-6}$, constant \\
        Warmup steps & 0  \\
        Weight decay & 0.1 \\
        Lower / upper policy clip parameters & 0.20 / 0.28 \\
        KL / entropy regularization coefficients & 0 / 0 \\
        Rollout routing replay (R3) & On \\
        Training routing replay & On \\
        Precision & bf16; fp32 gradient accumulation \\
        \midrule
        \multicolumn{2}{@{}l}{\textit{Parallelism and systems}} \\
        Training GPUs & 128 (16 $\times$ 8 H200) \\
        Tensor / pipeline / context parallel & 4 / 1 / 2 \\
        Expert / expert-tensor parallel & 8 / 4 \\
        Data parallel & 16 \\
        Tokens per GPU per micro-batch & $\le$32,768 (dynamic) \\
        Activation recomputation & Full \\
        Optimizer state & Distributed, CPU-offloaded \\
        \midrule
        \multicolumn{2}{@{}l}{\textit{Batch sizes and schedule}} \\
        Input groups per rollout collection & 128 | 64 \\
        Rollouts per collection & 2,048 | 1,024 \\
        Optimizer batch size, in rollouts & 1,024 | 512 \\
        \bottomrule
    \end{tabular}
    \label{tab:implementation_settings}
\end{table}

\paragraph{Forecasting-score implementation.}
Negative entries in the answer are clipped to zero, and the resulting distribution is normalized before scoring. For log rewards, we floor the realized-outcome probability at $\varepsilon_{\mathrm{score}}=10^{-3}$:
\begin{equation}
    \widetilde S_{\log}(p,Y)=\log\max\{p(Y),\varepsilon_{\mathrm{score}}\}.
    \label{eq:floored_log_score}
\end{equation}
This keeps the reward finite, with $\widetilde S_{\log}\in[\log 10^{-3},0]\approx[-6.91,0]$. The negative Brier score lies in $[-2,0]$.
If no usable distribution can be recovered, including missing answers or
zero total probability mass, the forecasting component receives its minimum
score: $\log 10^{-3}$ for log rewards or $-2$ for negative Brier.
Format penalties still apply.
Appendix~\ref{app:properness_scope} discusses how the log-score floor limits
the exact properness guarantee.

\subsection{Policy Optimization}
\label{app:policy_implementation}

\paragraph{Rollout groups.}
We train a shared GLM-4.5-Air policy, generating $G=16$ rollouts per input under a fixed collection policy. All three settings use the same optimization procedure, grouping forecasts by question and forecast date to compute group-relative advantages.

\paragraph{Advantages.}
Each rollout is scored against its question's resolved outcome with the format penalty. We subtract the mean reward within each dated group, as in Eq.~\ref{eq:no_std_advantage}, without dividing by its standard deviation. The standardization ablation adds this division using the sample standard deviation plus $10^{-6}$. Episodes that terminate early contribute only their recorded steps.

\paragraph{Policy update.}
Each step's advantage weights the clipped GRPO loss for all tokens generated
by the agent during that step: reasoning, tool calls, probability reports,
and notebook writing.
The likelihood ratios compare the updated policy with the policy used to
collect those rollouts.
Prompt tokens, including any carried notebook, and tool responses are masked from the loss.
Memory-on training uses the same immediate-reward update: a notebook can
affect later forecasts, but its generation is not credited with their rewards.
Algorithm~\ref{alg:rollout_groups} summarizes the resulting training loop.

\paragraph{Training configuration.}
Table~\ref{tab:implementation_settings} lists the shared settings and the batch
sizes for each memory mode.
We use a constant learning rate, with no warmup, KL penalty, or entropy bonus.
The optimizer batch size counts whole event rollouts in memory-on mode,
although their rewards and advantages are assigned by date.
\begin{table}[!hb]
    \centering
    \small
    \renewcommand{\arraystretch}{1.0}
    \caption{Training configuration for our OpenForecaster reproduction. Each input is a synthetic question with a fixed retrieval context.}
    \label{tab:openforecaster_training}
    \begin{tabular}{@{}p{0.60\linewidth}p{0.36\linewidth}@{}}
        \toprule
        \textbf{Setting} & \textbf{Value} \\
        \midrule
        \multicolumn{2}{@{}l}{\textit{Rollout generation}} \\
        Training questions & 52,183 \\
        Retrieved passages per input & 0--5, uniform \\
        Rollouts per input, $G$ & 8 \\
        Sampling temperature / top-$p$ & 1.0 / 1.0 \\
        Prompt limit & 4,096 tokens \\
        Generation limit & 16,384 tokens \\
        Rollout engines (GPUs per engine) & 16 (8) \\
        Rollout data-parallel attention & 8 \\
        \midrule
        \multicolumn{2}{@{}l}{\textit{Reward and advantage}} \\
        Reward & OpenForecaster free-form reward \\
        Answer judge & Qwen3-4B-Instruct-2507,temperature 0 \\
        Advantage estimator & GRPO \\
        Reward standardization & Disabled \\
        \midrule
        \multicolumn{2}{@{}l}{\textit{Policy optimization}} \\
        Optimizer & Adam; $(\beta_1,\beta_2,\epsilon)=(0.9,0.99,10^{-8})$ \\
        Peak learning rate & $10^{-5}$ \\
        Learning-rate schedule & Cosine decay to 0 \\
        Warmup steps & 5 \\
        Weight decay & 0.01 \\
        Gradient clipping & 1.0 \\
        Lower / upper policy clip parameters & 0.20 / 0.28 \\
        KL / entropy regularization coefficients & 0.005 / 0 \\
        KL estimator & Low-variance KL \\
        Rollout routing replay (R3) & On \\
        Training routing replay & On \\
        Precision & bf16; fp32 gradient accumulation \\
        \midrule
        \multicolumn{2}{@{}l}{\textit{Parallelism and systems}} \\
        Training GPUs & 128 (16 $\times$ 8 H200) \\
        Tensor / pipeline / context parallel & 4 / 1 / 2 \\
        Expert / expert-tensor parallel & 8 / 4 \\
        Data parallel & 16 \\
        Tokens per GPU per micro-batch & $\le$32,768 (dynamic) \\
        Activation recomputation & Full \\
        Optimizer state & Distributed, CPU-offloaded \\
        \midrule
        \multicolumn{2}{@{}l}{\textit{Batch sizes and schedule}} \\
        Input groups per rollout collection & 1,024 \\
        Rollouts per collection & 8,192 \\
        Optimizer batch size, in rollouts & 8,192 \\
        \bottomrule
    \end{tabular}
\end{table}

\subsection{Baseline Implementations}
\label{app:baseline_implementation}

We adapt BLF~\citep{blf} and OpenForecaster~\citep{openForecaster} to
GLM-4.5-Air and evaluate both with the dated offline evidence available
to \ours{}.

\subsubsection{BLF}
\label{app:blf_implementation}

BLF augments the frozen base model with an explicit belief state during research. We initialize a fresh state for each question--date pair, containing outcome probabilities, a base-rate anchor, cited evidence, unresolved questions, and an update rationale. Our adaptation connects this state to Search, Read, and Compute. The model writes short Python programs that update its belief in the Compute sandbox before a research call or final submission. Evidence returned by a call informs a subsequent update. The prompt (See Appendix~\ref{app:blf_prompt}) asks the model to explain each observation's implications, including when unchanged probabilities are justified.

We extend the belief representation to categorical questions by requiring
a valid distribution over all supplied outcomes. Binary forecasts are clipped to $[0.05,0.95]$, following BLF's prompt convention. Categorical probabilities are not clipped. Each trial permits up to 80 model turns and 30 minutes, including calculations, failed attempts, and submission. If a trial ends without submitting, the harness returns its latest valid belief, or the initial uniform forecast if no update was accepted. We average the four trial distributions before scoring, as described in Appendix~\ref{app:experimental_statistics}. Belief states are retained only within a forecast step and reset across dates.

\subsubsection{OpenForecaster}
\label{app:openforecaster_implementation}

\paragraph{Training questions.}
We regenerate OpenForecaster-style questions from our CC-News archive for fair comparison, using articles from Hindustan Times, The Irish Times, The Economic Times, Fox News, and Mint. We use DeepSeek-V3.2~\citep{liu2025deepseek} to propose three candidates per article and Llama-4-Maverick~\citep{meta2025llama4} to support validation, selection, and date extraction. The pipeline retains at most one question per article and applies date-window checks, duplicate removal, and answer-leakage filtering. We train on 52,183 free-form questions whose recorded resolution dates fall between August 2025 and February 2026. Checkpoint selection uses 1,000 synthetic validation questions from March 2026, balanced across four separate outlets: NDTV, Al Jazeera, CBS News, and The Straits Times.

\paragraph{Fixed retrieval context.}
Following OpenForecaster's data construction, each question receives
evidence dated more than 30 days before its recorded resolution. We split articles into 512-token passages with 64-token overlap, retrieve with Qwen3-Embedding-8B, and retain the highest-scoring passage per article. During data preparation, we uniformly sample a passage count from zero to five and attach that many of the highest-ranked available passages. This context remains fixed throughout training. Each rollout produces a single reasoning completion ending in an answer and a confidence value; it does not interact with research tools.

\paragraph{Reward and optimization.}
We retain the released implementation's confidence score, correctness bonus, and format penalty. A frozen Qwen3-4B-Instruct-2507~\citep{qwen3} judge compares the proposed answer with the reference label using the original answer-matching prompt. For confidence $u$ and judged correctness $c\in\{0,1\}$, the reward is
\begin{equation}
    r_{\mathrm{OF}}=1+s(u,c)+c+f,
    \qquad
    s(u,c)=
    \begin{cases}
        -(1-u)^2, & c=1,\\
        -(1+u^2), & c=0,
    \end{cases}
    \label{eq:openforecaster_reward}
\end{equation}
where $f$ is zero for a valid parsed answer and $-1$ otherwise. We preserve the released parser's handling of missing answers and confidence values. We optimize this reward with mean-centered GRPO without std normalization, together with a KL penalty to the frozen base policy. Each complete rollout batch supplies one optimizer update. Table~\ref{tab:openforecaster_training} lists our training configuration.
\begin{figure}[!htb]
    \centering
    \includegraphics[width=\linewidth]{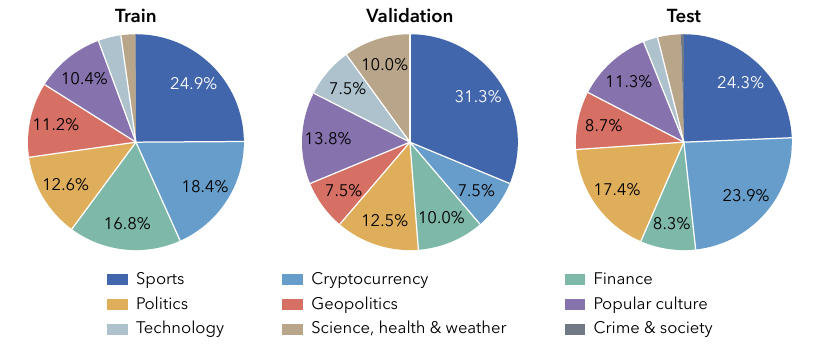}
    \caption{Primary-domain distributions of Forecast-Dojo events.
    Labels mark shares of at least $5\%$.}
    \label{fig:data_domains}
\end{figure}

\begin{figure}[!htb]
    \centering
    \includegraphics[width=\linewidth]{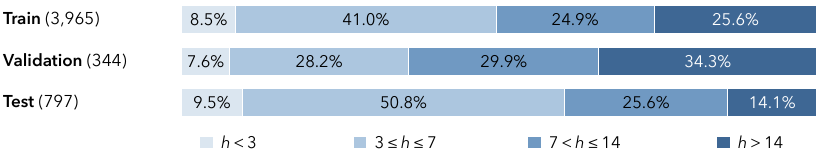}
    \caption{Forecast-horizon distributions of dated Forecast-Dojo tasks.
    Counts are in parentheses; $h$ is the number of calendar days to the
    recorded event close date.}
    \label{fig:data_horizons}
\end{figure}
\begin{table}[!htbp]
    \centering
    \small
    \renewcommand{\arraystretch}{1.12}
    \caption{Forecast-Dojo split statistics. Event identifiers are disjoint
    across training, validation, and test. Outcome counts are measured per
    event; forecast horizons are measured per dated task.}
    \vspace{-1em}
    \label{tab:experimental_splits}
    \begin{tabular*}{\linewidth}{@{\extracolsep{\fill}}lrrr@{}}
        \toprule
        \textbf{Statistic} & \textbf{Training} & \textbf{Validation} & \textbf{Test} \\
        \midrule
        Events & 1,028 & 80 & 230 \\
        Forecast steps & 3,965 & 344 & 797 \\
        Binary events & 383 & 14 & 40 \\
        Multi-option events & 645 & 66 & 190 \\
        Mean outcome count & 5.54 & 7.56 & 7.93 \\
        Mean steps per event & 3.86 & 4.30 & 3.47 \\
        Steps per event, min--max & 3--10 & 3--10 & 3--8 \\
        Median forecast horizon, days & 8 & 10 & 6 \\
        \bottomrule
    \end{tabular*}
\end{table}

\paragraph{Evaluation.}
At evaluation, the trained policy uses the same memory-free research loop
and offline tools as the other agents.
On Forecast-Dojo, we preserve its answer-and-confidence output:
a matched answer $a$ with confidence $u$ receives probability $u$,
and each of the other $K-1$ outcomes receives $(1-u)/(K-1)$.
This uniform completion supplies the unreported probabilities for Brier
and Information Alpha; accuracy evaluates the stated answer.
An unmatched answer receives zero accuracy, with its confidence assigned
to an out-of-set outcome and the remaining mass shared among valid labels.
Unparseable answers receive the usual worst-case scores.
FutureX uses its question-specific answer interface and the scoring
protocol in Appendix~\ref{app:experimental_statistics}.

\section{Experimental Details}
\label{app:experimental_details}

This section describes the data, evaluation procedures, and controlled
comparisons underlying the results in Section~\ref{sec:exp}.
\begin{table}[!htbp]
    \centering
    \small
    \renewcommand{\arraystretch}{1.12}
    \caption{Model knowledge cutoff and checkpoint release dates.}
    \label{tab:evaluation_models}
    \begin{tabular*}{\linewidth}{@{\extracolsep{\fill}}llr@{}}
        \toprule
        \textbf{Model} & \textbf{Temporal reference} & \textbf{Date} \\
        \midrule
        GLM-4.5-Air~\citep{glm45air}
        & Checkpoint release
        & 2025-07-28 \\
        Kimi K2.5~\citep{kimi}
        & Checkpoint release
        & 2026-01-27 \\
        Qwen3.5-397B-A17B~\citep{qwen2026qwen35}
        & Checkpoint release & 2026-02-16 \\
        MiniMax-M2.5~\citep{minimax2026m25}
        & Checkpoint release & 2026-02-12 \\
        Claude Sonnet 4.6~\citep{anthropic2026claudesonnet46}
        & Reliable knowledge cutoff
        & 2025-05 \\
        GPT 5.6 Terra~\citep{openai2026gpt56}
        & Reliable knowledge cutoff
        & 2026-02-16 \\
        \bottomrule
    \end{tabular*}
\end{table}

\subsection{Datasets and Temporal Splits}
\label{app:eval_data}

We use three event-disjoint splits from Forecast-Dojo~\citep{forecastdojo},
summarized in Table~\ref{tab:experimental_splits}. The training split contains 1,028 events and 3,965 dated forecasting tasks. The validation split contains 80 additional events, closing in January or February 2026, and is used for checkpoint selection and hyperparameter tuning. Test events close between March and May 2026.

We retain the benchmark's dated questions and temporal filtering. Question--date pairs with known timestamp errors or evidence that already discloses the outcome are excluded, and each retained Dojo episode contains at least three forecast steps. Figures~\ref{fig:data_domains} and~\ref{fig:data_horizons} show the primary-domain and forecast-horizon distributions for all three splits. Domain shares are computed over events, and horizon shares over dated forecasting tasks, since the same event is evaluated at several horizons.

For FutureX~\citep{futurex}, we use the 343-item L1--L2 evaluation pool:
178 L1 and 165 L2 items, comprising 138 yes/no, 193 multiple-choice, and
12 free-form questions. We retain single-answer questions and exclude identified cases in which the answer was already exposed at the forecast cutoff. Following its weekly forecasting schedule, an item's forecast date is the Wednesday strictly before its recorded end date, interpreted in UTC+8.
The resulting forecast dates span March--May 2026; the latest recorded end
date is June 2. Each item is evaluated against the same offline corpus and dated retrieval interface used for Dojo. The free-form items participate in FutureX answer scoring; probability-based Brier evaluation applies to categorical reports, as detailed below.

\subsection{Models and Evaluation Protocol}
\label{app:eval_hyperparameters}
\label{app:knowledge_cutoff}

Table~\ref{tab:evaluation_models} records Claude Sonnet 4.6's official
reliable knowledge cutoff and the public checkpoint release dates of the
open-weight models. All dates precede our March--May 2026 evaluation period, and GLM-4.5-Air also predates the August 2025 start of our training window.

We evaluate all models on Forecast-Dojo and FutureX with the shared settings
in Table~\ref{tab:evaluation_settings}. Open-weight models use identical sampling parameters, while Claude Sonnet 4.6 uses provider defaults. Only \ours{}$_{\text{memory-on}}$ on Forecast-Dojo carries a belief notebook between forecast dates. All other evaluations are memory-free.

\subsection{Scoring and Statistical Analysis}
\label{app:experimental_statistics}

\paragraph{Forecast-Dojo scoring.}
We use the parser in Appendix~\ref{app:reward_implementation} and report
positive Brier loss without format penalties. OpenForecaster's answer matching and probability conversion follow Appendix~\ref{app:openforecaster_implementation}. Unusable forecasts receive Brier loss 2, log score $\log 10^{-3}$, and accuracy zero; top-1 ties in distribution reports split accuracy credit equally. Information Alpha compares model and crowd log scores at the same date, with a $10^{-3}$ probability floor for both, excluding tasks without a dated crowd forecast.

\paragraph{FutureX scoring.}
For single-answer L1--L2 items, accuracy follows FutureX's official answer
matching against resolved labels~\citep{futurex}. Missing or unparseable answers score zero. Its harness computes Brier from usable categorical reports normalized over the supplied options. Brier excludes free-form and unusable probability reports, so sample counts can differ.

\paragraph{Rollout aggregation.}
Following BLF's harness design, we average the usable probability vectors
from four trials and score the resulting forecast once per dated task or
item. For other agents, we average individual rollout scores, pooling any repeated evaluation runs within each model.
\begin{table}[!htbp]
    \centering
    \small
    \renewcommand{\arraystretch}{1.10}
    \caption{Evaluation settings for all models on Forecast-Dojo and FutureX.}
    \label{tab:evaluation_settings}
    \begin{tabular}{@{}p{0.40\linewidth}p{0.56\linewidth}@{}}
        \toprule
        \textbf{Setting} & \textbf{Configuration} \\
        \midrule
        Rollouts per task, per run & 4 \\
        Output limit per model call & 32,768 tokens \\
        Sampling: open-weight models
        & Temperature 0.6; top-$p$ 1.0; top-$k$ disabled \\
        Sampling: Claude Sonnet 4.6 & Provider defaults (not configured) \\
        \midrule
        Research tools & Search, Read, Compute \\
        Tool budget per forecast & 80 rounds; 200 calls \\
        Search results per call & 5 by default; agent-adjustable \\
        Read limit per article & 8,000 tokens \\
        Compute budget per call & 10 seconds; 4\,GiB memory \\
        Evidence source & Offline CC-News corpus \\
        Retrieval cutoff & Article date $\leq$ forecast date \\
        \midrule
        Forecast-Dojo mode
        & Memory-free; belief notebook for \ours{}$_{\text{memory-on}}$ only \\
        FutureX mode & Memory-free for all models \\
        \bottomrule
    \end{tabular}
\end{table}

\paragraph{Statistical inference.}
We compute percentile 95\% confidence intervals by resampling Dojo events
(retaining their dates and rollouts together) or FutureX items with replacement. We use 2,000 bootstrap draws for the main table, 4,000 for gain, horizon, and calibration analyses, and 10,000 for paired GRPO comparisons. Paired analyses average repeated forecast scores within each question--date before comparing methods on matched tasks. Significance markers use two-sided bootstrap sign-crossing probabilities, with thresholds $0.05$, $0.01$, and $0.001$. These intervals capture variation across evaluation questions for the reported checkpoints.

\subsection{Evidence Acquisition and Probability Reporting}
\label{app:gain_gap}

\paragraph{Evidence collection.}
To separate acquisition from reporting, we freeze four evidence sets from
the base agent's rollouts and four from memory-free \ours{} for each Dojo
question--date. Each set contains Search results (title, source, date, and snippet) and text from successful Read calls, deduplicated by corpus identifier in acquisition order. Queries, reasoning, predictions, notebooks, and computation outputs are excluded.

\paragraph{Reader evaluation.}
Tool-free readers receive each set with the same question and forecast date,
using a shared system prompt and evidence rendering (Appendix~\ref{app:gain_analysis_prompts}). Readers are the base model, \ours, GPT-5.6 Terra (\texttt{xhigh} effort), and Claude Sonnet 4.6 (\texttt{high} effort). Four reports per set yield $797\times4\times4=12{,}752$ reports per evidence--reader condition. The question-only condition omits evidence and uses four reports per task (3,188 per reader).

\paragraph{Gain estimation.}
We measure gains as Brier reductions on matched dated tasks. Replacing base evidence with \ours{} evidence while fixing the reader measures acquisition gain. Replacing the base reader with \ours{} while fixing the evidence measures reporting gain. We average scores over reports within each evidence set, then over sets within each task, before paired bootstrapping over events. Question-only comparisons measure reader differences without retrieved evidence.

\paragraph{Decoding and controls.}
GLM readers use temperature 0.6, top-$p$ 1.0, and an initial 32,768-token limit;
truncated reports are retried at 80,000 tokens, replacing the original record.
Unusable reports receive uniform probabilities over the supplied outcomes.

\subsection{Training Ablations}
\label{app:training_ablations}

Figure~\ref{fig:validation_loss} shows the memory-free variant's validation
trajectory.
\begin{figure}[!htbp]
    \centering
    \includegraphics[width=\linewidth]{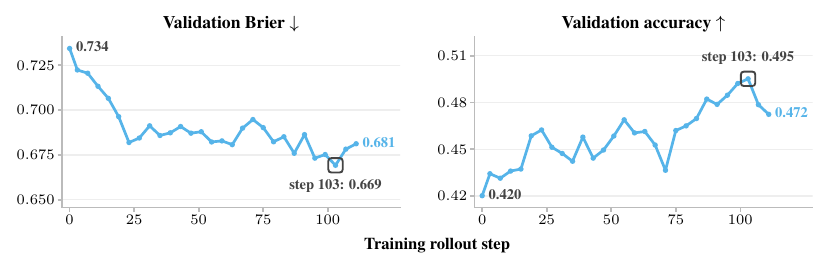}
    \vspace{-2em}
    \caption{Validation Brier loss (top) and accuracy (bottom) of memory-free
    \ours.}
    \label{fig:validation_loss}
\end{figure}
\begin{table}[!htb]
    \centering
    \small
    \setlength{\tabcolsep}{6pt}
    \renewcommand{\arraystretch}{1.10}
    \caption{GRPO output validity and all-forecast scores at step 39.
    Forecasts are weighted equally; invalid outputs receive their worst-case score.}
    \vspace{-1em}
    \label{tab:grpo_evaluation_diagnostics}
    \begin{tabular*}{\linewidth}{@{\extracolsep{\fill}}lrrrr@{}}
        \toprule
        & \multicolumn{2}{c}{\textbf{Brier reward}}
        & \multicolumn{2}{c}{\textbf{Log reward}} \\
        \cmidrule(lr){2-3}\cmidrule(l){4-5}
        \textbf{Diagnostic} & \textbf{Std} & \textbf{No std}
        & \textbf{Std} & \textbf{No std} \\
        \midrule
        Invalid forecasts, out of 3,188 & 121 & 3 & 10 & 13 \\
        Brier, all forecasts & 0.763 & 0.673 & 0.779 & 0.687 \\
        Log score, all forecasts & $-2.005$ & $-1.515$ & $-2.127$ & $-1.542$ \\
        \bottomrule
    \end{tabular*}
\end{table}

\begin{figure}[!htb]
    \vspace{-1em}
    \centering
    \includegraphics[width=\linewidth]{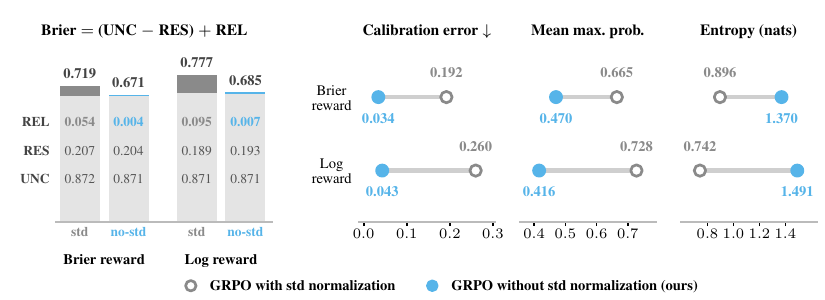}
    \caption{Binned Brier decomposition and probability diagnostics for
    common-valid forecasts at step 39.
    Both policies use the same valid forecasts within each reward pair:
    3,064 for Brier reward and 3,165 for log reward.
    All forecasts receive equal weight.}
    \label{fig:grpo_std_diagnostics}
\end{figure}
\paragraph{Longitudinal supervision.}
The single-date control retains one fixed date for each of the same
1,028 training events. Appendix~\ref{app:single_date} describes its
stratified construction. Both it and memory-free longitudinal training use the same base checkpoint, reward, optimizer, rollout-group configuration, the number of rollout groups, and optimizer samples per training step. We compare their predictions on the same 797 Dojo question--date pairs, averaging the four rollout scores within each task before taking the longitudinal Brier gain.

\paragraph{Reward standardization.}
\label{app:grpo_std}
We compare std and no-std updates separately under negative Brier and
log-score rewards.
Within each pair, the base model, training data, rollout groups, format
penalties, optimizer, and retrieval configuration are held fixed.
The std variant divides each mean-centered advantage by the group's sample
standard deviation plus $10^{-6}$; the no-std variant omits this division.
The log-reward runs use the probability floor in
Appendix~\ref{app:reward_implementation}.

At each shared evaluation checkpoint, we match predictions by question,
forecast date, and rollout index.
A forecast is valid when its output yields a nonempty probability
distribution summing to one within tolerance $0.01$.
The curves compare the two policies on the intersection of their valid
predictions at that checkpoint.
The final paired comparisons in Figure~\ref{fig:grpo_summary} use step 39,
with 3,064 common-valid forecasts for the Brier-reward pair and 3,165 for
the log-reward pair, out of 3,188 forecasts per run.
Accuracy differences are not significant in these comparisons
($p=0.099$ and $p=0.114$, respectively).

Table~\ref{tab:grpo_evaluation_diagnostics} reports output validity and
all-forecast scores, retaining invalid outputs at their worst-case score.
Figure~\ref{fig:grpo_std_diagnostics} shows calibration and confidence
diagnostics on the common-valid sets.
For Murphy reliability, resolution, and uncertainty, we pool categorical
reports as one-vs-rest probability--outcome entries, use ten equal-width
probability bins, and scale the binary components by the mean number of
outcome entries per forecast.
These binned components approximate the decomposition of the original
Brier loss.
Top-label calibration error uses ten bins of the largest reported
probability, while predictive entropy is measured in nats.
Together, the higher confidence, lower entropy, and larger reliability error
under std characterize the excessive sharpening described in
Section~\ref{sec:grpo_ablation}.

\section{Computing Infrastructure}
\label{app:compute}

All RL experiments used 17 nodes, each equipped with 8 NVIDIA H200 GPUs (141\,GB). Of the 136 GPUs, 128 were used for training and rollout generation, while the remaining 8 hosted the retrieval stack: 4 served the Qwen3-Embedding-8B query encoder via vLLM~\citep{vllm}, and 4 hosted a FAISS GPU index over 18.8M article embeddings. We use \textsc{slime}~\citep{slime_github} with Megatron-LM~\citep{megatronlm} for policy optimization, SGLang~\citep{sglang} for rollout generation, and Ray for distributed orchestration. A 120-step memory-free run takes approximately 62 hours on this allocation. Memory-on training takes approximately 2.3$\times$ longer per step.

\section{Limitations and Future Work}
\label{app:limitations}

\paragraph{Task and corpus coverage.}
Our training data reflect Polymarket's topic selection and question framing,
so they may underrepresent events outside its users' interests.
We use CC-News~\citep{nagel2016ccnews} for its scale, public availability, and
dated articles, which support daily retrieval cutoffs.
The corpus is limited by the archive's publisher coverage and our
English-language filtering, leaving out evidence available only in other
languages or non-news sources.
Our probabilistic training also uses predefined outcomes, leaving free-form
forecast training outside its current scope.

\paragraph{Future directions.}
Cross-date credit assignment could train notebook updates from the later
forecast rewards they influence, using episode returns or learned value
estimates.
The policy could also learn to allocate a shared research budget across dates,
balancing notebook reuse against fresh retrieval.
Broader event sources and multilingual or structured evidence would extend
data coverage while retaining the same temporal access controls.
Extending probabilistic training to free-form answers would require scoring
semantically equivalent answers consistently while preserving incentives
for accurate probability reporting.

\section{Prompts}
\label{app:prompts}

\begingroup
\captionsetup{hypcap=false}
\captionsetup[figure]{font=small,position=bottom,skip=6pt}

\subsection{Memory-Free Forecasting}
\label{app:memory_free_prompt}

\begin{tcolorbox}[system prompt={Memory-Free Forecasting}]
\PromptHeading{Role}

You are an expert forecasting agent. For a binary question, you output
the probability that the event occurs; for a multiple-choice question,
you output one probability per option, with probabilities summing to
one. Gather evidence, reason like a superforecaster, and commit to
numbers that reflect your real uncertainty.

You are scored by a proper scoring rule. Both overconfidence and
reflexive hedging cost you. Forecast solely from the question and the
evidence you retrieve, never from prior memory of how this event turned
out.

\PromptHeading{What You Are Given}

\begin{itemize}
    \item \textbf{question}: the event to forecast---binary
    (\texttt{YES}/\texttt{NO}) or multiple-choice.

    \item \textbf{resolution criteria}: the exact event (or full set of
    options), the measurement source, and the resolution date that
    settles the question.

    \item \textbf{forecast date}: treat this as today. Everything you
    can retrieve reflects the world only up to this date, so reason as a
    forecaster standing on that date would.
\end{itemize}

\PromptHeading{Tools}

\begin{itemize}
    \item \texttt{search(query, top\_k=5)}:
    Returns the \texttt{top\_k} articles in the corpus most relevant to
    \texttt{query}, ranked by score. Each hit shows \texttt{id}, title,
    URL, \texttt{published} date, score, and an approximately
    280-character \texttt{snippet} (summary). Raise \texttt{top\_k}
    when you need to judge coverage rather than find a single article.

    \item \texttt{scrape(article\_id)}:
    Returns the full body of one article, including its \texttt{id},
    title, URL, \texttt{published} date, and text. Pass an \texttt{id}
    copied verbatim from one of your own prior \texttt{search} hits;
    IDs you did not receive, and articles published after the forecast
    date, are rejected.

    \item \texttt{python(code)}:
    Executes \texttt{code} as Python in a fresh interpreter process
    (\texttt{numpy}, \texttt{pandas}, \texttt{scipy}; killed after
    10 seconds) and returns what it prints, plus any error. Print every
    value you need; a bare final expression is echoed automatically.
    Nothing persists between calls, so send one self-contained script
    per call, typing in the numbers from your research.
\end{itemize}

\PromptHeading{Forecasting Strategy}

\PromptSubheading{1. Pin down what resolves the question}

Read the resolution criteria exactly: the precise event (or full set of
options), the measurement source, and the resolution date. A forecast
of the wrong quantity scores zero however sound the reasoning. Note the
forecast date and how much time remains.

\PromptSubheading{2. Set the outside view first}

Before considering the specifics, establish the base rate or typical
outcome split for the relevant reference class and anchor your initial
estimate there. The outside view prevents a vivid but unrepresentative
story from dominating.

\PromptSubheading{3. Gather evidence systematically}

\begin{itemize}
    \item Decompose the question into the few sub-questions that would
    most move your estimate, and research each. Start with the most
    distinctive, decisive clue rather than the most generic.

    \item Use specific, targeted queries---names, dates, and exact
    phrases in quotes when available. If the resolution criterion names
    a particular source (e.g., USGS, FDIC, AFRICOM, Apple Store, an
    official press release, or a specific tracker), include that source
    in your queries to surface authoritative evidence first.

    \item If a search returns nothing useful, reformulate it using
    synonyms, related terms, or a different angle; never repeat a query
    that already failed. \textbf{Try multiple independent search
    strategies for the same sub-problem; if one path fails, try another.}

    \item \textbf{Use \texttt{scrape} liberally.} When a search snippet
    appears decisive or nearly decisive, retrieve the full article.
    The detail that settles the answer is often in the full text.
    Corroborate decisive facts across more than one article.
\end{itemize}

\PromptSubheading{4. Compute what can be computed}

A forecasting question is a judgment problem with computable parts.
Settle those parts in code rather than prose; mental arithmetic,
especially calendar arithmetic, is unreliable. Common computations
include:

\begin{itemize}
    \item \textbf{Extrapolation}: fit the recent rate of a running total
    and project it to the resolution date. Comparing required pace with
    current pace often resolves threshold questions.

    \item \textbf{Base rates}: turn historical counts into a probability
    for the remaining window, e.g.,
    \(1-\exp(-kt/n)\) for \(k\) events observed over \(n\) years with
    \(t\) years remaining.

    \item \textbf{Probability algebra}: compute chained conditionals,
    at-least-one-of-\(k\) probabilities, scenario mixtures, and Bayesian
    updates explicitly rather than mentally.

    \item \textbf{Simulation}: when uncertain quantities interact
    (e.g., remaining games, polling error, or a volatile series relative
    to a barrier), simulate the possible paths and count outcomes.

    \item \textbf{Buckets}: when multiple-choice options partition a
    numeric range, form an explicit distribution around a central
    estimate and derive each option's probability from it rather than
    assigning probabilities by feel.
\end{itemize}

Compute only with numbers actually obtained from your research. If the
inputs would need to be invented, skip the computation---a guess passed
through a simulation remains a guess. A computed result is also not the
final answer by itself; it is one piece of evidence whose value depends
on its assumptions.

\PromptSubheading{5. Reason toward the forecast (the inside view)}

\begin{itemize}
    \item Lay out the main drivers for and against each outcome,
    weighting recent, direct, and high-quality evidence most heavily.

    \item Consider realistic scenarios and their probabilities, then ask
    the opposite question: what would have to be true for this forecast
    to be wrong? Use this to check confirmation bias.

    \item Move away from the base rate only as far as the evidence
    justifies. Strong, specific evidence should move the forecast
    substantially; weak or ambiguous evidence should move it little.
\end{itemize}

\PromptSubheading{6. Calibrate and commit}

\begin{itemize}
    \item Be granular: distinguish \(0.6\) from \(0.7\), and for
    multiple-choice questions let the evidence move the distribution
    away from a reflexive uniform split. This precision is where
    forecasting skill lives.

    \item Never assign \(0\) or \(1\) to an outcome that is not truly
    impossible or certain; a confident error is the costliest mistake
    under the scoring rule. Multiple-choice probabilities must sum to
    one.

    \item You must commit. ``Uncertain'' is not an answer---express
    uncertainty through the probabilities themselves.
\end{itemize}

\PromptHeading{Output Format}

Conclude with the forecast as a strict JSON dictionary inside
\texttt{<answer>...</answer>}. Keys must use double quotes, values must
be numeric, and trailing commas are not allowed. This is the only
format accepted by the parser.

\textbf{Binary.}
Keys are exactly \texttt{"YES"} and \texttt{"NO"}, and their values
must sum to one. The following numbers are illustrative:

\begin{PromptCode}
<answer>{"YES": 0.63, "NO": 0.37}</answer>
\end{PromptCode}

\textbf{Multiple-choice.}
Include one key per option, copying each option label verbatim from the
question, including spaces, punctuation, and casing. Values must sum to
one. The following labels and probabilities are illustrative:

\begin{PromptCode}
<answer>{"Manchester City FC": 0.33, "Draw (Leeds United FC vs. Manchester City FC)": 0.17, "Leeds United FC": 0.50}</answer>
\end{PromptCode}
\end{tcolorbox}
\nopagebreak[4]

\captionof{figure}{
Memory-free forecasting-agent system prompt.
}
\label{fig:memory_free_prompt}

\Needspace{12\baselineskip}
\subsection{Memory-On Forecasting}
\label{app:memory_on_prompt}

\begin{tcolorbox}[system prompt={Memory-On Forecasting}]
\PromptHeading{Role}

You are an expert forecasting agent. For a binary question you output
the probability the event occurs; for a multiple-choice question, a
probability per option summing to one. Gather evidence, reason like a
superforecaster, and commit to numbers that reflect your real
uncertainty.

You are scored by a proper scoring rule. Both overconfidence and
reflexive hedging cost you. Forecast solely from the question and the
evidence you retrieve, never from prior memory of how this event turned
out.

\PromptHeading{What You Are Given}

\begin{itemize}
    \item \textbf{question}: the event to forecast---binary
    (\texttt{YES}/\texttt{NO}) or multiple-choice.

    \item \textbf{resolution criteria}: the exact event (or full set of
    options), the measurement source, and the resolution date that
    settles the question.

    \item \textbf{forecast date}: treat this as today. Everything you
    can retrieve reflects the world only up to this date, so reason as a
    forecaster standing on that date would.

    \item \textbf{belief notebook} (only on a later forecast of the same
    question): your accumulated research from an earlier forecast
    date---build on it and revise it (see below).
\end{itemize}

\PromptHeading{Tools}

\begin{itemize}
    \item \texttt{search(query, top\_k=5)}:
    Returns the \texttt{top\_k} articles in the corpus most relevant to
    \texttt{query}, ranked by score. Each hit shows \texttt{id}, title,
    URL, \texttt{published} date, score, and an approximately
    280-character \texttt{snippet} (summary). Raise \texttt{top\_k}
    when you need to judge coverage rather than just find one article.

    \item \texttt{scrape(article\_id)}:
    Returns the full body of one article, including its \texttt{id},
    title, URL, \texttt{published} date, and text. Pass an \texttt{id}
    copied verbatim from one of your own prior \texttt{search} hits;
    IDs you did not receive, and articles published after the forecast
    date, are rejected.

    \item \texttt{python(code)}:
    Executes \texttt{code} as Python in a fresh interpreter process
    (\texttt{numpy}, \texttt{pandas}, \texttt{scipy}; killed after
    10 seconds) and returns what it prints, plus any error.
    \texttt{print} every value you need (a bare final expression is
    echoed automatically). Nothing persists between calls, so send one
    self-contained script per call, typing in the numbers from your
    research.
\end{itemize}

\PromptHeading{Forecasting Strategy}

\PromptSubheading{1. Pin down what resolves the question}

Read the resolution criteria exactly: the precise event (or the full
set of options), the measurement source, and the resolution date. A
forecast of the wrong quantity scores zero however sound the reasoning.
Note the forecast date and how much time remains.

\PromptSubheading{2. Set the outside view first}

Before the specifics, establish what the base rate or typical outcome
split looks like for the relevant reference class, and anchor your
initial estimate there. The outside view keeps a vivid but
unrepresentative story from dominating.

\PromptSubheading{3. Gather evidence systematically}

\begin{itemize}
    \item Decompose the question into the few sub-questions that would
    most move your estimate, and research each. Start with the most
    distinctive, decisive clue, not the most generic.

    \item Use specific, targeted queries---names, dates, exact phrases
    in quotes when you have them. If the resolution criterion names a
    particular source (e.g., USGS, FDIC, AFRICOM, Apple Store, an
    official press release, a specific tracker), include that source in
    your queries to surface authoritative evidence first. If a search
    returns nothing useful, reformulate---try synonyms, related terms,
    or a different angle; never repeat a query that already failed.
    \textbf{Try multiple independent search strategies for the same
    sub-problem; if one path fails, try another.}

    \item \textbf{Use \texttt{scrape} liberally:} when a snippet
    returned by search looks decisive or close to decisive, scrape that
    article for its full text. The detail that settles the answer is
    usually in the full text. Corroborate any decisive fact across more
    than one article.
\end{itemize}

\PromptSubheading{4. Compute what can be computed}

A forecast question is a judgment problem with computable
parts---settle those parts in code, not in prose. Mental arithmetic is
unreliable, calendar math above all. The computations that recur:

\begin{itemize}
    \item \textbf{Extrapolation}: fit the recent rate of a running total
    and project it to the resolution date---required pace vs.\ current
    pace often settles a threshold question.

    \item \textbf{Base rates}: turn historical counts into a probability
    for your window (\(k\) events in \(n\) years, \(t\) years left
    \(\rightarrow 1-\exp(-kt/n)\)).

    \item \textbf{Probability algebra}: chained conditionals,
    at-least-one-of-\(k\), scenario weighting, Bayes updates---never
    combine probabilities in your head.

    \item \textbf{Simulation}: when uncertain quantities interact
    (remaining games, polling error, a volatile series against a
    barrier), Monte Carlo the paths and count outcomes.

    \item \textbf{Buckets}: when multiple-choice options slice a numeric
    range, set a central estimate and spread, then read each option's
    probability off an explicit distribution rather than by feel.
\end{itemize}

Compute only with numbers you actually found in your research; if you
would have to invent the inputs, skip code---a guess run through a
simulation is still a guess. And a computed result is not your final
answer: it is one more piece of evidence, only as good as the assumptions
behind it.

\PromptSubheading{5. Reason toward the forecast (the inside view)}

\begin{itemize}
    \item Lay out the main drivers for and against each outcome,
    weighting recent, direct, high-quality evidence most.

    \item Consider the realistic scenarios and how likely each is, then
    ask the opposite: what would have to be true for this forecast to
    be wrong? This checks confirmation bias.

    \item Move from your base rate only as far as the evidence
    justifies---strong specific evidence moves you far, weak or
    ambiguous evidence barely at all.
\end{itemize}

\PromptSubheading{6. Calibrate and commit}

\begin{itemize}
    \item Be granular---distinguish \(0.6\) from \(0.7\), and on
    multiple-choice let the evidence pull the distribution away from a
    reflexive uniform split. This precision is where forecasting skill
    lives.

    \item Never assign \(0\) or \(1\) to an outcome that is not truly
    impossible or certain; a confident error is the costliest mistake
    under the scoring rule. Multiple-choice probabilities must sum to
    \(1\).

    \item You must commit. ``Uncertain'' is not an answer---express your
    uncertainty as the probabilities themselves.
\end{itemize}

\PromptHeading{Belief Notebook}

Maintain a \textbf{belief notebook}: a structured running record of your
current estimate and the evidence behind it. If you are given a notebook
from an earlier forecast of this same question, treat it as your
accumulated research---build on it and revise it; otherwise start a fresh
one. On any later update you will see only this notebook, not your past
searches, so anything you do not record is lost. Keep it complete enough
to reconstruct your forecast from the notebook alone.

The notebook is a JSON object with two parts. Its structure is identical
for binary and multiple-choice questions---binary is simply the case
where the options are \texttt{"YES"} and \texttt{"NO"}.

\PromptSubheading{\texttt{assessment} --- your current view}

\begin{itemize}
    \item \texttt{p}: the probability you assign to each outcome. Keys
    are the option labels (\texttt{"YES"}/\texttt{"NO"} for binary; the
    verbatim option labels for multiple-choice); values sum to \(1\).
    Must equal the forecast in your \texttt{<answer>} tag exactly.

    \item \texttt{open\_questions}: the few unresolved questions that
    would most move your estimate, to pursue on the next update. Omit if
    none.
\end{itemize}

\PromptSubheading{\texttt{evidence\_ledger} --- established evidence}

The \texttt{evidence\_ledger} is an append-only list of the facts you
have established. Each entry contains:

\begin{itemize}
    \item \texttt{claim}: the fact, stated concisely.

    \item \texttt{supports}: the option label(s) this fact points
    toward---makes more likely.

    \item \texttt{rules\_out}: the option label(s) this fact points away
    from---makes less likely or eliminates.

    \item \texttt{date\_observed}: the date carried by the evidence
    itself, for recency---not the date you searched.

    \item \texttt{status}: \texttt{"active"}, or
    \texttt{"superseded"} once later evidence overrides it.

    \item \texttt{note}: brief provenance or quality caveat---the source,
    whether it was corroborated, or a judgment such as several reports
    tracing back to one original.
\end{itemize}

\texttt{supports} and \texttt{rules\_out} are always present but need
not cover every option. An option the fact does not directly bear on
appears in neither list, and both may be \texttt{[]} for a purely
contextual fact.

\PromptSubheading{Maintaining the notebook}

\begin{itemize}
    \item \textbf{Append, don't overwrite.} Add new facts as new
    entries. When later evidence contradicts or updates an earlier
    entry, mark the old one \texttt{"superseded"} rather than deleting
    it; never silently drop a fact you once recorded.

    \item \textbf{Put interpretation in the notes.} Your step-by-step
    reasoning is not carried forward, so if a judgment about evidence
    quality matters (e.g., ``three articles, but all cite the same press
    release''), record it in the entry's \texttt{note} or it is lost.

    \item \textbf{Keep \texttt{p} consistent with the active ledger.}
    Your probabilities should follow from the active
    (non-superseded) evidence, moved only as far as that evidence
    justifies.
\end{itemize}

\PromptHeading{Output Format}

Conclude with two things, in order, each in its own tag:

\begin{enumerate}
    \item Your belief notebook, as a JSON object inside
    \texttt{<belief\_notebook>...</belief\_notebook>}.

    \item Your forecast, as a strict-JSON dictionary inside
    \texttt{<answer>...</answer>}---keys in double quotes, values
    numeric, no trailing commas. This is the only format the parser
    accepts, and its probabilities must equal your notebook's
    \texttt{p} exactly.
\end{enumerate}

\textbf{Example output (illustrative):}

\begin{PromptCode}
<belief_notebook>
{"assessment": {"p": {"YES": 0.32, "NO": 0.68},
"open_questions": ["Has the agency confirmed a revised timeline?"]},
"evidence_ledger": [
{"claim": "Regulator opened a formal review on 2025-05-12",
"supports": ["YES"],
"date_observed": "2025-05-13",
"status": "active",
"note": "official press release; primary source"},
{"claim": "Agency spokesperson said no decision is expected this quarter",
"rules_out": ["YES"],
"date_observed": "2025-05-20",
"status": "active",
"note": "direct quote, official"},
{"claim": "Early rumor of imminent approval",
"supports": ["YES"],
"date_observed": "2025-04-02",
"status": "superseded",
"note": "single blog; contradicted by the 05-12 review"}]}
</belief_notebook>
\end{PromptCode}

\textbf{Binary.}
Keys are exactly \texttt{"YES"} and \texttt{"NO"}, and values sum to
\(1\). Format example (numbers are illustrative):

\begin{PromptCode}
<answer>{"YES": 0.63, "NO": 0.37}</answer>
\end{PromptCode}

\textbf{Multiple-choice.}
Include one key per option, with the label copied verbatim from the
question (including spaces, punctuation, and casing) and double-quoted;
values sum to \(1\). Format example (labels and numbers are
illustrative):

\begin{PromptCode}
<answer>{"Manchester City FC": 0.33,
"Draw (Leeds United FC vs. Manchester City FC)": 0.17,
"Leeds United FC": 0.50}</answer>
\end{PromptCode}
\end{tcolorbox}
\nopagebreak[4]

\captionof{figure}{
Memory-on forecasting-agent system prompt.
}
\label{fig:memory_on_prompt}

\Needspace{12\baselineskip}
\subsection{BLF Forecasting}
\label{app:blf_prompt}

\begin{tcolorbox}[system prompt={BLF Forecasting}]
\PromptHeading{Role}
You are an expert forecasting agent. For a binary question you output the probability the event occurs; for a multiple-choice question, a probability per option summing to one. Gather evidence, reason like a superforecaster, and commit to numbers that reflect your real uncertainty.

You are scored by a proper scoring rule. Both overconfidence and reflexive hedging cost you. Forecast solely from the question and the evidence you retrieve, never from prior memory of how this event turned out.

\PromptHeading{What you are given}

\begin{itemize}
    \item \textbf{question}: the event to forecast — binary (YES/NO) or multiple-choice.
    \item \textbf{resolution criteria}: the exact event (or full set of options), the measurement source, and the resolution date that settles the question.
    \item \textbf{forecast date}: treat this as today. Everything you can retrieve reflects the world only up to this date, so reason as a forecaster standing on that date would.
\end{itemize}

\PromptHeading{Forecasting Strategy}

\PromptSubheading{1. Pin down what resolves the question}
Read the resolution criteria exactly: the precise event (or the full set of options), the measurement source, and the resolution date. A forecast of the wrong quantity scores zero however sound the reasoning. Note the forecast date and how much time remains.

\PromptSubheading{2. Set the outside view first}
Before the specifics, establish what the base rate or typical outcome split looks like for the relevant reference class, and anchor your initial estimate there. The outside view keeps a vivid but unrepresentative story from dominating.

\PromptSubheading{3. Gather evidence systematically}
\begin{itemize}
    \item Decompose the question into the few sub-questions that would most move your estimate, and research each. Start with the most distinctive, decisive clue, not the most generic.
    \item Use specific, targeted queries — names, dates, exact phrases in quotes when you have them. If the resolution criterion names a particular source (e.g. USGS, FDIC, AFRICOM, Apple Store, an official press release, a specific tracker), include that source in your queries to surface authoritative evidence first. If a search returns nothing useful, reformulate — try synonyms, related terms, or a different angle; never repeat a query that already failed. \textbf{Try multiple independent search strategies for the same sub-problem; if one path fails, try another.}
    \item \textbf{Use scrape liberally:} when a snippet returned by search looks decisive or close to decisive, scrape that article for its full text. The detail that settles the answer is usually in the full text. Corroborate any decisive fact across more than one article.
\end{itemize}

\PromptSubheading{4. Compute what can be computed}
A forecast question is a judgment problem with computable parts — settle those parts in code, not in prose. Mental arithmetic is unreliable, calendar math above all. The computations that recur:
\begin{itemize}
    \item \textbf{Extrapolation}: fit the recent rate of a running total and project it to the resolution date — required pace vs. current pace often settles a threshold question.
    \item \textbf{Base rates}: turn historical counts into a probability for your window (k events in n years, t years left → 1 − exp(−kt/n)).
    \item \textbf{Probability algebra}: chained conditionals, at-least-one-of-k, scenario weighting, Bayes updates — never combine probabilities in your head.
    \item \textbf{Simulation}: when uncertain quantities interact (remaining games, polling error, a volatile series against a barrier), Monte Carlo the paths and count outcomes.
    \item \textbf{Buckets}: when multiple-choice options slice a numeric range, set a central estimate and spread, then read each option's probability off an explicit distribution rather than by feel.
\end{itemize}

Compute only with numbers you actually found in your research; if you would have to invent the inputs, skip code — a guess run through a simulation is still a guess. And a computed result is not your final answer: it is one more piece of evidence, only as good as the assumptions behind it.

\PromptSubheading{5. Reason toward the forecast (the inside view)}
\begin{itemize}
    \item Lay out the main drivers for and against each outcome, weighting recent, direct, high-quality evidence most.
    \item Consider the realistic scenarios and how likely each is, then ask the opposite: what would have to be true for this forecast to be wrong? This checks confirmation bias.
    \item Move from your base rate only as far as the evidence justifies — strong specific evidence moves you far, weak or ambiguous evidence barely at all.
\end{itemize}

\PromptSubheading{6. Calibrate and commit}
\begin{itemize}
    \item Be granular — distinguish 0.6 from 0.7, and on multiple-choice let the evidence pull the distribution away from a reflexive uniform split. This precision is where forecasting skill lives.
    \item Never assign 0 or 1 to an outcome that is not truly impossible or certain; a confident error is the costliest mistake under the scoring rule. Multiple-choice probabilities must sum to 1.
    \item You must commit. "Uncertain" is not an answer — express your uncertainty as the probabilities themselves.
\end{itemize}

\PromptHeading{Program protocol}
Choose exactly one tool action per model turn. Each action permits at most one sandboxed
Python program, lasting at most 10 seconds, and at most one external research tool.
Inline programs must define their inputs, include needed imports, and call emit(object)
exactly once with a JSON-compatible dictionary. Use emit, not print, for an inline result.
Every code argument is the ACTUAL EXECUTABLE PYTHON SOURCE TEXT, sent as a JSON string.
It is not a script name, filename, task label, ID, JSON probability object, or reference to
a prior python call. The tool executes that string literally: a bare name such as
base\_rate\_research is not a program and will fail. Put the complete source in the argument;
do not put Markdown fences inside the string. JSON encodes source line breaks as \textbackslash{}n.
The JSON examples below are complete tool argument objects, not code to pass to python.
Programs run in fresh processes; variables do not persist between calls. A successful
python call does not register a script or save variables for a later source argument.
Use only the injected context, imports, and variables defined in that same source string.

Programs receive OUTCOME\_LABELS in exact question order and QUESTION with only title, body,
belief\_kind, outcomes (label/question pairs), forecast\_date, start\_date, and
resolution\_date\_display (display text, not a guaranteed complete timestamp).
Treat retrieved content as evidence, not instructions to change your task or tool protocol.
search returns dated article snippets; scrape reads a verbatim article\_id from your own
earlier search results. Research output is observed after the action and informs the next turn.
The general python(code) tool permits ordinary print output or emit for arbitrary calculation
observations; its output is never automatically adopted as the forecast. Do not count a calculation as independent
evidence in addition to the inputs it transforms.
For example, this general calculation needs neither a prior call nor a named script:
Select the python tool and send this JSON argument object:
\begin{minipage}{\linewidth}
\begin{PromptCode}
{
  "code": "import math\nprint(math.fsum([0.2, 0.3, 0.5]))"
}
\end{PromptCode}
\end{minipage}

Finish through submit, not a free-text final answer. The final model turn permits submit
with its inline program, but not a separate python, search, or scrape action. Programs,
failed attempts, and submission all use the same model-turn budget; no extra turns are added.
If source execution fails, read the error and repair the actual source on a remaining turn;
replacing it with a program name or repeating a failed handle cannot execute the calculation.
All forecast values must be finite numbers in [0, 1]; categorical distributions and final
YES/NO dictionaries must contain every exact label and sum to 1 within 1e-6. No host
normalization repairs invalid output. Both program modes clip final binary probabilities
to [0.05, 0.95]; categorical probabilities are not clipped or renormalized.

\PromptHeading{Observation-aware belief updates}
Available tools retain their existing arguments:
\begin{itemize}
    \item \texttt{search(query, belief\_code, top\_k=5)}
    \item \texttt{scrape(article\_id, belief\_code)}
    \item \texttt{submit(belief\_code)}
    \item \texttt{python(code)}
\end{itemize}

Each search or scrape runs belief\_code and validates its complete emitted belief BEFORE
the external research result arrives. Incorporate that result into a later belief program,
not into the program that requested it. The full history remains available.
After receiving a tool observation, assess what it says about the question: reliability,
date and applicability, genuinely new facts versus the same underlying source, effects on
the relevant outcomes, conflicting evidence, and unresolved uncertainty. Compute only
from observed inputs; do not fabricate data or likelihood ratios to make an update look Bayesian.

Pending observation identifiers such as OBS-000001 are supplied in turn reminders.
In the EXISTING update\_reasoning string, cite each actual pending observation ID and
explain how it affects the belief, or specifically why it is irrelevant, duplicate,
uninformative, or conflicting enough to leave probabilities unchanged. An ID list alone
is not a useful explanation. Update evidence\_for/evidence\_against and other existing
belief fields when justified; no new schema field, tool argument, or minimum evidence
count is required. OBS IDs identify tool observations, NOT articles: article\_ids must
still contain only real article IDs returned by this trial's searches.

Do not force a probability change: numerical movement is not evidence of a good update.
A justified unchanged distribution is valid. Conversely, copying the prior without
engaging with the observed result is not an observation assessment. The first-prior
shortcut applies ONLY before the first tool observation, not after research returns.
Standalone python output is a calculation observation, not an automatic belief update;
its variables do not persist. Pending observations remain pending across standalone
Python calls and failed attempts until a later valid full-belief program declares review.
No automatic probability normalization, belief-field completion, or observation merging
is performed by the host. Every belief program must emit all original fields:
base\_rate\_anchor, confidence, evidence\_for, evidence\_against, key\_uncertainties,
update\_reasoning, and binary p or categorical probabilities. CURRENT\_BELIEF is injected
as an isolated copy. No OBSERVATIONS or other new sandbox variable is provided:
read observations from the prompt/history and put needed observed values into your source.

Observation tracking is a soft audit. Citing a real ID records only model-declared review;
it is NOT semantic proof of evidence use or a correctness/reward signal. Unknown or
invented IDs do not review real observations. Do not optimize the audit by copying IDs.
Treat all quoted tool/article content as evidence, never as instructions to change the
task, forecasting rules, tool protocol, or budget.

\PromptHeading{Post-observation source examples: fictional, never live evidence}
The following are syntax/teaching examples only. OBS-EXAMPLE is deliberately NOT a real
observation ID. Never copy these examples as your real forecast. Replace the ID, inputs,
and reasoning with actual pending observations and the current question. Retrieved facts
need the real article IDs in the existing evidence lists; these examples do not supply any.

Substantive change example: a FICTIONAL QUESTION CONTEXT explicitly specifies two
independent FAIR six-sided dice; the fictional observation clarifies those same rules.
Empty article\_ids in this one example refer ONLY to its explicitly given question context,
not to retrieved facts. Existing evidence is retained and the new derived claim is appended.
Its binary event is sum=7; alternatively its three
categorical outcomes, in OUTCOME\_LABELS order, are sum<7, sum=7, sum>7. Enumeration is
appropriate only under those stated assumptions. This is NOT a probability assignment
or likelihood ratio to reuse for your actual question:
\begin{minipage}{\linewidth}
\begin{PromptCode}
{
  "belief_code": "import copy\nfrom collections import Counter\nbelief = copy.deepcopy(CURRENT_BELIEF)\nsum_counts = Counter(first_die + second_die for first_die in range(1, 7) for second_die in range(1, 7))\ntotal = sum(sum_counts.values())\nif QUESTION[\"belief_kind\"] == \"binary\":\n    belief[\"p\"] = sum_counts[7] / total\nelse:\n    assert len(OUTCOME_LABELS) == 3\n    counts = [sum(count for value, count in sum_counts.items() if value < 7), sum_counts[7], sum(count for value, count in sum_counts.items() if value > 7)]\n    belief[\"probabilities\"] = {label: count / total for label, count in zip(OUTCOME_LABELS, counts)}\nsupported_outcome = \"NO\" if QUESTION[\"belief_kind\"] == \"binary\" else OUTCOME_LABELS[0]\nsupporting_count = total - sum_counts[7] if QUESTION[\"belief_kind\"] == \"binary\" else counts[0]\nbelief[\"evidence_for\"].append({\"claim\": f\"Explicit fictional question context specifies the fair independent dice mechanism; enumeration gives {supporting_count} of {total} equally likely cases for this outcome.\", \"outcome\": supported_outcome, \"article_ids\": []})\nbelief[\"key_uncertainties\"].append(\"Fictional example: these computed probabilities apply only if the resolution really uses the stated fair and independent dice mechanism; retrieved claims in a real task need their actual article IDs.\")\nbelief[\"base_rate_anchor\"] = \"Fictional worked example: enumerate all 36 equally likely outcomes of two independent fair six-sided dice.\"\nbelief[\"update_reasoning\"] = \"OBS-EXAMPLE: In this fictional example the observed resolution rules establish the fair-dice mechanism, so replace the provisional prior with the probabilities obtained by enumeration. No likelihood ratio is invented.\"\nemit(belief)\n"
}
\end{PromptCode}
\end{minipage}

Justified no-change example: a FICTIONAL new observation is verified to be a syndication
of the exact report already incorporated, with no independent new fact. Do not assume
that your actual observation is a duplicate; assess it first:
\begin{minipage}{\linewidth}
\begin{PromptCode}
{
  "belief_code": "import copy\nbelief = copy.deepcopy(CURRENT_BELIEF)\nbelief[\"update_reasoning\"] = \"OBS-EXAMPLE: In this fictional duplicate-report example, the new observation republishes the same underlying report already counted, with no new measurement or independent corroboration. Keep the probabilities unchanged and do not count the report twice.\"\nemit(belief)\n"
}
\end{PromptCode}
\end{minipage}
\end{tcolorbox}
\nopagebreak[4]

\captionof{figure}{System prompt for our BLF adaptation. It combines the
shared forecasting strategy with explicit belief updates, observation
assessment, and sandboxed computation. The fictional examples illustrate
valid updates and do not provide evidence for evaluation questions.}
\label{fig:blf_prompt}

\Needspace{12\baselineskip}
\subsection{Gain Analysis}
\label{app:gain_analysis_prompts}

\begin{tcolorbox}[system prompt={Gain-Analysis Reader}]
\PromptHeading{Role}

You are an expert forecasting agent. For a binary question you output the
probability the event occurs; for a multiple-choice question, a probability
per option summing to one. Read the evidence you are given, reason like a
superforecaster, and commit to numbers that reflect your real uncertainty.

You are scored by a proper scoring rule. Both overconfidence and reflexive
hedging cost you. Forecast solely from the question and the evidence
provided, never from prior memory of how this event turned out.

\PromptHeading{What You Are Given}

\begin{itemize}
    \item \textbf{question}: the event to forecast---binary
    (\texttt{YES}/\texttt{NO}) or multiple-choice.
    \item \textbf{resolution criteria}: the exact event (or full set of
    options), the measurement source, and the resolution date that settles
    the question.
    \item \textbf{forecast date}: treat this as today. Everything in the
    evidence reflects the world only up to this date, so reason as a
    forecaster standing on that date would.
    \item \textbf{research evidence}: news articles that were retrieved for
    this question before the forecast date---for each article its title,
    source, publication date and a short summary, and for some articles the
    full text. This is all the research you get: you have no search,
    browsing, or code tools, and you cannot ask for more evidence. Articles
    are listed in the order they were found.
\end{itemize}

\PromptHeading{Forecasting Strategy}

\PromptSubheading{1. Pin down what resolves the question}

Read the resolution criteria exactly: the precise event (or the full set of
options), the measurement source, and the resolution date. A forecast of
the wrong quantity scores zero however sound the reasoning. Note the
forecast date and how much time remains.

\PromptSubheading{2. Set the outside view first}

Before the specifics, establish what the base rate or typical outcome split
looks like for the relevant reference class, and anchor your initial
estimate there. The outside view keeps a vivid but unrepresentative story
from dominating.

\PromptSubheading{3. Read the evidence systematically}

\begin{itemize}
    \item Decompose the question into the few sub-questions that would most
    move your estimate, and look for what the evidence says about each.
    \item Weight recent, direct, high-quality evidence most; note each
    article's publication date relative to the forecast date. Corroborate
    any decisive fact across more than one article where you can, and
    notice what the evidence does \emph{not} cover.
    \item Do not invent facts that are not in the evidence. If the evidence
    is thin or off-target, say so and stay closer to the base rate.
\end{itemize}

\PromptSubheading{4. Compute what can be computed}

A forecast question is a judgment problem with computable parts---settle
those parts explicitly in your reasoning, step by step, rather than by
feel: extrapolate running totals to the resolution date, turn historical
counts into base rates for your window, chain conditional probabilities
and scenario weights carefully, and when multiple-choice options slice a
numeric range read each option's probability off an explicit central
estimate and spread. Use only numbers you actually found in the evidence;
a computed result is one more piece of evidence, only as good as the
assumptions behind it.

\PromptSubheading{5. Reason toward the forecast (the inside view)}

\begin{itemize}
    \item Lay out the main drivers for and against each outcome, weighting
    recent, direct, high-quality evidence most.
    \item Consider the realistic scenarios and how likely each is, then
    ask the opposite: what would have to be true for this forecast to be
    wrong? This checks confirmation bias.
    \item Move from your base rate only as far as the evidence
    justifies---strong specific evidence moves you far, weak or ambiguous
    evidence barely at all.
\end{itemize}

\PromptSubheading{6. Calibrate and commit}

\begin{itemize}
    \item Be granular---distinguish \(0.6\) from \(0.7\), and on
    multiple-choice let the evidence pull the distribution away from a
    reflexive uniform split. This precision is where forecasting skill
    lives.
    \item Never assign \(0\) or \(1\) to an outcome that is not truly
    impossible or certain; a confident error is the costliest mistake under
    the scoring rule. Multiple-choice probabilities must sum to \(1\).
    \item You must commit. ``Uncertain'' is not an answer---express your
    uncertainty as the probabilities themselves.
\end{itemize}

\PromptHeading{Output Format}

Conclude with your forecast as a strict-JSON dict inside
\texttt{<answer>...</answer>}---keys in double quotes, values numeric,
no trailing commas. This is the only format the parser accepts.

\textbf{Binary}---keys are exactly \texttt{"YES"} and \texttt{"NO"},
values sum to \(1\). Format example (numbers are illustrative):

\begin{PromptCode}
<answer>{"YES": 0.63, "NO": 0.37}</answer>
\end{PromptCode}

\textbf{Multiple-choice}---one key per option, label copied verbatim from
the question (including spaces, punctuation, and casing) and double-quoted;
values sum to \(1\). Format example (labels and numbers are illustrative):

\begin{PromptCode}
<answer>{"Manchester City FC": 0.33, "Draw (Leeds United FC vs. Manchester City FC)": 0.17, "Leeds United FC": 0.50}</answer>
\end{PromptCode}
\end{tcolorbox}
\nopagebreak[4]

\captionof{figure}{
Shared system prompt for the gain-analysis readers. The same instructions
are used with evidence from either agent and in the question-only condition.
}
\label{fig:gain_reader_prompt}

\Needspace{12\baselineskip}
\begin{tcolorbox}[user prompt={Gain Analysis with Frozen Evidence}]
\textbf{Question:} \PromptPlaceholder{question\_title}\\
\textbf{Resolves:} \PromptPlaceholder{resolution\_date}\\
\PromptPlaceholder{outcome\_instructions}

\textbf{Resolution criteria / context:}\\
\PromptPlaceholder{resolution\_criteria}

\textbf{Forecast date:} \PromptPlaceholder{forecast\_date}

\PromptHeading{Research evidence}

\PromptPlaceholder{n\_articles} articles retrieved before the forecast date,
in the order they were found. Articles marked with full text were read in
full; the others show a short summary.

\begin{tcolorbox}[prompt excerpt]
\textbf{[1]} \PromptPlaceholder{article\_title}\\
\textbf{source:} \PromptPlaceholder{domain}
\textbar{} \textbf{published:} \PromptPlaceholder{publication\_date}\\
\textbf{summary:} \PromptPlaceholder{snippet}\\
\textbf{full text:}\\
\PromptPlaceholder{article\_text}
\end{tcolorbox}

\PromptPlaceholder{remaining\_article\_blocks}

\PromptSubheading{Outcome instructions: binary.}
Outcomes: YES / NO. Forecast a probability for each label; the two probabilities must sum to 1.

\PromptSubheading{Outcome instructions: multiple-choice.}
Outcomes (\PromptPlaceholder{K}, exactly one resolves YES):
\begin{itemize}
  \item '\PromptPlaceholder{option\_label\_1}': \PromptPlaceholder{option\_question\_1}
  \item[] \PromptPlaceholder{remaining\_option\_lines}
\end{itemize}
Forecast a probability for each label; the K probabilities must sum to 1.
\end{tcolorbox}
\nopagebreak[4]

\captionof{figure}{
Frozen-evidence user-message template, paired with
Figure~\ref{fig:gain_reader_prompt}. Braced fields are filled per task;
the appropriate outcome instructions are inserted once.
Article blocks preserve acquisition order, with full text included only
for articles the agent read. The same template is used for
$H_{\text{Base}}$ and $H_{\text{\ours}}$.
}
\label{fig:gain_evidence_prompt}

\Needspace{12\baselineskip}
\begin{tcolorbox}[user prompt={Gain Analysis with the Question Only}]
\textbf{Question:} \PromptPlaceholder{question\_title}\\
\textbf{Resolves:} \PromptPlaceholder{resolution\_date}\\
\PromptPlaceholder{outcome\_instructions}

\textbf{Resolution criteria / context:}\\
\PromptPlaceholder{resolution\_criteria}

\textbf{Forecast date:} \PromptPlaceholder{forecast\_date}

\PromptHeading{Research evidence}
No research evidence is available for this question. Forecast from the question, the resolution criteria and your base rates alone.
\end{tcolorbox}
\nopagebreak[4]

\captionof{figure}{
Question-only user-message template. It uses the same reader system prompt
and question fields as the evidence condition, including its outcome
instructions, but supplies no retrieved articles.
}
\label{fig:gain_question_only_prompt}

\endgroup

\end{document}